\documentclass{article} 
\usepackage{iclr2027_conference,times}

\usepackage{amsmath,amsfonts,bm}

\def\eqref#1{equation~\ref{#1}}

\def\1{\bm{1}}

\DeclareMathAlphabet{\mathsfit}{\encodingdefault}{\sfdefault}{m}{sl}
\SetMathAlphabet{\mathsfit}{bold}{\encodingdefault}{\sfdefault}{bx}{n}

\DeclareMathOperator*{\argmin}{arg\,min}

\usepackage[dvipsnames]{xcolor}
\definecolor{fav_color}{RGB}{0,128,0}
\usepackage[
colorlinks=true,
allcolors = fav_color,
citecolor=fav_color
]{hyperref}
\usepackage[utf8]{inputenc} 
\usepackage[T1]{fontenc}    
\usepackage{url}            
\usepackage{booktabs}       
\usepackage{amsfonts}       
\usepackage{nicefrac}       
\usepackage{microtype}      
\usepackage{xcolor}         
\usepackage[most]{tcolorbox}
\usepackage{bm}
\usepackage{amsmath}
\usepackage{amssymb}
\usepackage{mathtools}
\usepackage{amsthm}
\usepackage{graphicx}
\usepackage{floatrow}
\usepackage{subcaption}
\usepackage{caption}
\usepackage{wasysym}
\usepackage{bbm}
\usepackage{multirow}
\usepackage{enumitem}

\usepackage{algorithm}
\usepackage{algorithmic}
\usepackage{amsmath}
\usepackage{amssymb}
\usepackage{mathtools}
\usepackage{amsthm}
\usepackage[capitalize,noabbrev]{cleveref}

\theoremstyle{plain}
\newtheorem{theorem}{Theorem}[section]
\newtheorem{proposition}[theorem]{Proposition}
\newtheorem{lemma}[theorem]{Lemma}

\theoremstyle{definition}

\theoremstyle{remark}

\tcolorboxenvironment{theorem}{
  enhanced,
  breakable,
  colframe=fav_color,
  colback=fav_color!8,     
  boxrule=1pt,
  toprule=1pt,
  left=4pt, right=4pt, top=4pt, bottom=4pt,
  arc=2pt 
}

\tcolorboxenvironment{proposition}{
  enhanced,
  breakable,
  colframe=fav_color,
  colback=fav_color!8,    
  boxrule=1pt,
  toprule=1pt,
  left=4pt, right=4pt, top=4pt, bottom=4pt,
  arc=2pt 
}

\title{An Informational Curse of Horizon in \\Goal-Conditioned Policy Learning}

\author{John L. Zhou\thanks{Equal contribution} \quad Yuxuan Dong\footnotemark[1] \quad Jonathan C. Kao\\
University of California, Los Angeles\\
\texttt{john.ly.zhou@gmail.com}, \texttt{dyx723@g.ucla.edu} \\
}

\iclrfinalcopy 
\begin{document}

\maketitle
\vspace{-0.1cm}
\begin{abstract}
The difficulty of learning goal-reaching policies is often attributed to a ``curse of horizon'' that manifests as bias accumulation in temporal-difference backups and noisy advantage estimates. In this work, we identify an additional \textit{informational} curse of horizon in goal-conditioned policy learning, where increasing the goal relabeling horizon can significantly reduce policy generalization and performance. Through a series of controlled experiments with oracle planners, we decouple the goal horizons sampled during training from those that the policy is asked to reach at test time. Even when evaluated only on a sequence of nearby subgoals, goal-conditioned behavioral cloning (BC) policies suffer from severe, training horizon-dependent performance degradation that is mitigated by reinforcement learning (RL) objectives. We explain this phenomenon as a horizon-dependent decrease in the conditional mutual information between actions and hindsight-relabeled goals, and find empirically that both BC and RL policies trained on longer-horizon goals exhibit a shift in sensitivity from goal to state information, as measured by the policy's input Jacobians. Motivated by this observation, we find that distilling the input Jacobians of short-horizon policies into long-horizon policies yields significant performance gains, especially in combinatorial manipulation tasks. Taken together, our results highlight goal relabeling horizon as an important consideration when learning generalist policies from offline data.
\end{abstract}

\section{Introduction}

Learning to reach any state from any other state is an elegantly simple problem that allows an agent to acquire diverse behaviors from data, without supervision \citep{kaelbling_learning_1993}. A core component of standard goal-conditioned policy learning involves hindsight goal relabeling \citep{andrychowicz_hindsight_2017}, which allows an agent to learn from arbitrary experiences regardless of the actor's original intentions. However, simply providing a policy with auxiliary input in the form of goal observations or natural language instructions does not necessarily produce a policy that is capable of following such instructions. A goal-reaching policy must learn to meaningfully utilize information about the desired goal to generate actions, and we argue that this dependence can only emerge when knowing the goal provides meaningful \textit{additional} information about the policy learning objective.

We formalize this additional goal-related information as the conditional mutual information (MI) $\mathcal I(A; G \mid S)$ between the goal $G$ and action $A$ conditioned on the current state $S$, which we refer to as the \textbf{action-goal mutual information}. Under standard hindsight relabeling, action-goal MI is non-increasing with the relabeling horizon [\Cref{prop:mi_decay}]. This points to a distinct \textit{informational} curse of horizon in goal-conditioned policy learning: increasing the horizon of relabeled training goals not only introduces well-characterized learning challenges \citep{park_horizon_2025}, but may also alter \textit{how} the policy is learned. Specifically, as distant goals provide progressively less information about which action should be taken, the policy may converge to a solution that is globally less sensitive to goal information---even when conditioned on nearby goals.

However, disentangling these two effects is surprisingly difficult in standard goal-reaching benchmarks \citep{park_ogbench_2024}. Long-horizon evaluation goals are typically more challenging and require longer sequences of correct actions, while policies must also be trained on distant goals to simply keep them in-distribution at evaluation time. This makes it virtually impossible to distinguish failures due to increased task difficulty from failures due to weaker goal conditioning. 

To isolate the effect of training horizon $H_\mathrm{train}$ on performance, we construct a series of controlled experiments using an oracle hierarchical planner that ``prompts'' a learned, goal-conditioned policy with optimal, fixed-distance (sub)goals at test time. We find that behavioral cloning (BC) policies degrade severely as $H_\mathrm{train}$ increases, even when the distance of goals presented at test time $H_\mathrm{plan}$ remains fixed. On the other hand, reinforcement learning (RL) policies are substantially more robust to increased training goal horizons, although this robustness is closely tied to the RL objective (weighted BC or reparameterized policy gradient) and policy class (Gaussian or conditional flow matching) [\Cref{subsec:empirical_oracle}]. These results suggest that poor policy performance is not only due to higher task difficulty, but also a degradation in the learned policy's ability to achieve simpler goals.

Our theoretical analysis predicts that longer-horizon goals provide less action-relevant information, but how might this be reflected in a policy parameterized by a neural network? In our empirical studies, we use the Jacobian norm of the action with respect to the state or goal as a local measure of policy sensitivity. For certain parameterizations, e.g., fixed-variance Gaussians, the Jacobian norm can be directly related to Fisher information [Box \ref{box:fisher}]. Across BC and RL objectives, as well as Gaussian and flow-matching policy classes, we find that increasing the training goal horizon systematically shifts sensitivity away from goal conditioning towards the current state. Furthermore, explicitly reducing the action-goal MI by training on some proportion of arbitrary, unrelated goal states produces a similar reduction in goal sensitivity.

To determine if there is a causal relationship between goal sensitivity and performance, we then ask if preserving the sensitivity structure of policies trained exclusively on short-horizon goals can improve the performance of policies trained on longer-horizon goals. We propose \textbf{Sobolev bootstrapping}, a simple modification of existing policy bootstrapping methods that additionally distills the \textit{input Jacobians} of a short-horizon teacher policy into a long-horizon student policy. This simple intervention significantly boosts performance in a suite of \textbf{16} diverse locomotion and manipulation environments, with especially large gains in combinatorial manipulation tasks.

\section{Preliminaries}
We consider the problem of offline {goal-conditioned RL}, described by a Markov decision process (MDP) $\mathcal{M} = (\mathcal{S}, \mathcal{A}, \mathcal{R}, \mathcal{P}, \mathcal G)$ where $\mathcal{S}$ is the state space, $\mathcal{A}$ the action space, $\mathcal{R} : \mathcal S \times \mathcal S \rightarrow \mathbb R$ the goal-conditioned reward function, $\mathcal{P}: \mathcal S \times \mathcal A \to \mathcal S$ the transition function, and $\mathcal G$ the goal space. In this work, we assume the goal space is equivalent to the state space.

In the \textit{offline} setting, we are given a reward-free dataset of environment interactions $\mathcal D$ and asked to learn a policy $\pi(\cdot \mid s, g)$ to reach a specified goal $g$. At test time, the learned policy is deployed into the environment for a set of evaluation goals drawn from some goal distribution $p(g)$. Then, we can write the objective as
\begin{align}
  \mathcal J(\pi)=\mathbb{E}_{g \sim p(g), \tau \sim p^\pi(\tau)}\left[\sum_{t=0}^\infty \gamma^t r\left(s_t, g\right)\right],
\end{align}
where $p^\pi(\tau)$ is the distribution of trajectories generated by the policy $\pi$ and the transition function $\mathcal P$ during (online) evaluation. Importantly, we do not assume access to $p(g)$ and must learn a general policy that is capable of reaching any goal encountered in the dataset.

\subsection{Input Jacobians and policy sensitivity}
\label{subsec:input_jacobian_def}
Our analyses draw inspiration from prior work that studies input gradients, defined simply as the gradient of the model output with respect to its inputs, from the perspective of classifier robustness in computer vision applications. In our work, we extend these ideas to the parameters of a policy action distribution $\pi$. For exposition, we first study policies modeled as fixed-variance Gaussian distributions, where $a \sim \mathcal N(\mu_\theta(s, g), \mathbf I)$. Differentiating this output with respect to the state or goal yields the policy's \textit{state and goal Jacobians},
\begin{align}
  \label{eq:input_jacobian}
  J_{\pi}(s) = \frac{\partial\mu_\theta( a\mid s, g)}{\partial s} \quad \text{and} \quad J_{\pi}(g) = \frac{\partial\mu_\theta( a\mid s, g)}{\partial g}.
\end{align}
Prior work has established that properties of input gradients such as the magnitude of the $\ell_1$ norm and degree of alignment with human perception strongly correlate with model robustness \citep{tsipras_robustness_2018,ganz_perceptually_2023}, and either regularizing the input gradient norm \citep{drucker_improving_1992,ross_improving_2018,rodriguez-munoz_characterizing_2025} or distilling input gradients from a robust teacher through Sobolev training \citep{czarnecki_sobolev_2017,chan_what_2020} can significantly improve generalization under perturbations.

\section{An Informational Curse of Horizon}
\label{sec:curse}

\begin{figure}[t]
  \begin{center}
  \includegraphics[width=1.0\linewidth]{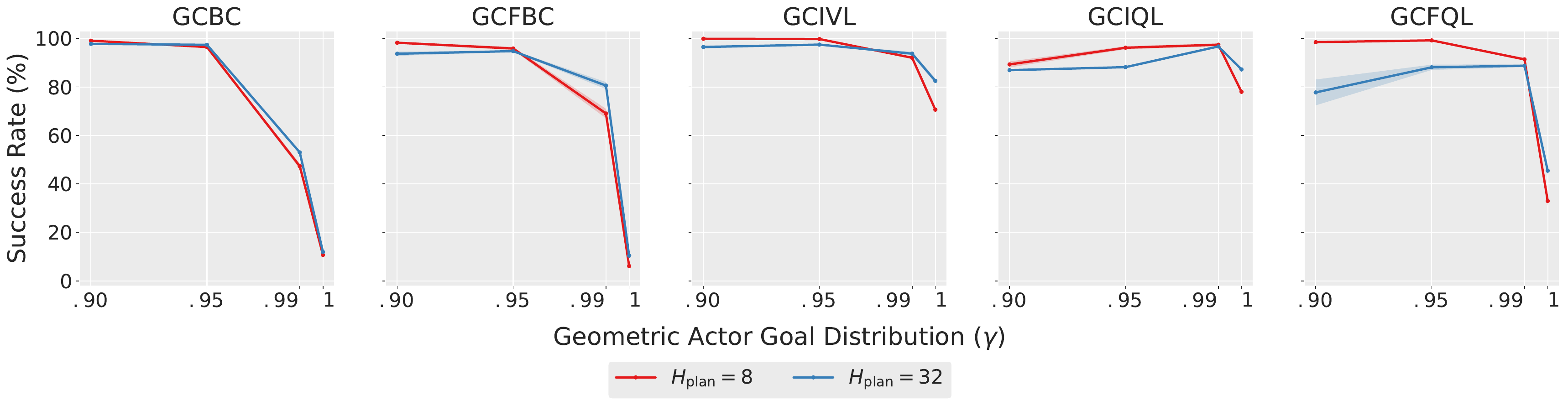}
  \caption{\textbf{Training on more distant relabeled goals harms \textit{nearby} goal-reaching capabilities}. We record the success rates across policies trained on a geometric distribution of relabeling goal offsets with discount $\gamma$, where $\gamma=1$ corresponds to a uniform distribution over future states. At evaluation, policies are given optimal $H_\mathrm{plan}$-step subgoals from \textit{oracle planners}. Success rates are averaged over the last three checkpoints, three seeds, and four environments: \texttt{cube-double}, \texttt{cube-triple}, \texttt{cube-quadruple}, and \texttt{scene}.}
  \label{fig:geom_horizon_oracle_success}
  \end{center}
\end{figure}

Horizon has many opportunities to influence the performance of a goal-reaching policy in subtly different ways. Longer evaluation horizons allow compounding error to accumulate over time, reducing the success rate of a policy with a fixed per-step error rate $\epsilon$ \citep{ross_efficient_2010}. Furthermore, directly learning to reach distant goals is simply \textit{harder} in the first place, due to noisy advantage estimates and compounding bias in value learning [Appendix \ref{subsec:curse}].

We identify a third challenge in which increasing the horizon of relabeled training goals, $H_\mathrm{train}$, can degrade the policy's ability to reach nearby goals as well. Said differently, the difficulties induced by long-horizon training \textit{are not confined to reaching distant goals}. 

To study this effect, we introduce a third horizon: the planning horizon, $H_\mathrm{plan}$, which separates the total evaluation horizon $H_\mathrm{eval}$ from the horizon of the goals presented to the learned policy. Using an oracle planner, we decouple $H_\mathrm{plan}$ from $H_\mathrm{train}$ and $H_\mathrm{eval}$. This reveals severe, training horizon-dependent degradation for goal-conditioned BC policies and a milder, albeit objective-dependent form of the same phenomenon in RL policies. We provide an information-theoretic explanation for these findings and study how it manifests in policies parameterized by neural networks.

\subsection{How does horizon affect performance?}
\label{subsec:empirical_oracle}
To isolate policy quality from task difficulty, we use an oracle planner that presents the policy with optimal subgoals at a fixed $H_\mathrm{plan}$ while leaving the underlying evaluation task unchanged. $H_\mathrm{plan}$ is similar in spirit to the subgoal horizon used in hierarchical approaches, but here we use it only at evaluation rather than learning a separate high-level policy.

We vary the distribution over training goal horizons in three complementary ways. First, we sample future in-trajectory goals from a geometric distribution with increasing discount factors $\gamma$, which assign more probability mass to longer horizons while maintaining the same goal support set [Figures \ref{fig:geom_horizon_oracle_success} \& \ref{fig:all_geom_horizon_oracle_success}]. Then, we directly control the range of goal distances the actor sees by truncating uniform distributions to different maximum offsets [Figure \ref{fig:uniform_horizon_oracle_success_all}]. Finally, we mix in arbitrary, unrelated dataset states as goals, directly weakening the predictive relationship between actions and goals [Figure \ref{fig:all_random_oracle_success}].

Across both the fixed- and geometric-horizon experiments, we observe a clear degradation in nearby goal-reaching competence as BC policies are trained on increasingly distant relabeled goals. Because the evaluation task and $H_\mathrm{plan}$ are fixed, this effect cannot simply be attributed to the policy being asked to reach more distant goals. Instead, it is clear that increasing $H_\mathrm{train}$ changes the quality of the learned actor itself. RL policies exhibit a milder version of the same phenomenon, where the degree of robustness to horizon depends heavily on the objective. 

Our results demonstrate more nuance than a monotonic relationship between training goal horizon and performance. Geometric relabeling can retain strong performance even with exposure to long-horizon goals, and providing unrelated dataset states as goals results in consistent performance decline across all algorithms. Taken together, our experiments suggest that temporal or spatial training goal distance alone cannot fully account for policy degradation. To resolve this discrepancy, we hypothesize that the informational relationship between actions and goals is the critical factor. Next, we formalize this intuition and investigate how training on extended horizons alters the learned policy for nearby goals.

\subsection{Why should training horizon matter?}
\label{sec:mi_theory}

To understand why training on extended horizons can degrade nearby goal-reaching performance [\Cref{subsec:empirical_oracle}], we theoretically analyze the divergence between a goal-conditioned policy and its unconditional behavioral prior, focusing on how sources of goal-dependent action information change with relabeling horizon. To make the analysis tractable, we focus on the KL-regularized objective corresponding to \citet{peng_advantage-weighted_2019}. We use uppercase letters (e.g., $S_t$) to denote random variables, and lower-case letters (e.g., $s_t$) to denote their instantiations. The full proof is provided in \Cref{appendix:proof}.

\begin{theorem}[Divergence Bound from the Goal-Agnostic Prior]
\label{thm:tvd_bound}
Let $\pi^{*}(a \mid s, g) \propto p_{\mathrm{HR}}(a \mid s, g) \exp(Q(s, a, g) / \alpha)$ be the target policy induced by a KL-regularized policy update, where $\alpha$ is a temperature parameter and $p_{\mathrm{HR}}$ is the distribution of tuples $(S, A, G)=\left(S_t, A_t, S_{t+K}\right)$ induced by the dataset trajectory distribution and a hindsight relabeling rule $K \sim q_{\mathrm{HR}}$. Let $\beta(a \mid s)$ be the goal-agnostic behavior policy induced by the dataset, and let the critic value span be $\Delta_Q(s, g) = \sup_{a} Q(s, a, g) - \inf_{a} Q(s, a, g)$ over the support of $p_{\mathrm{HR}}$. Then, the expected squared total variation distance between the two policies is bounded by:
$$\mathbb{E}_{S, G}\left[D_{\mathrm{TV}}\left(\pi^{*}(\cdot \mid S, G), \beta(\cdot \mid S)\right)^2\right] \leq \frac{\mathbb{E}_{S, G}\left[\Delta_Q(S, G)^2\right]}{8 \alpha^2} + I_{\mathrm{HR}}(A ; G \mid S)$$
\end{theorem}

This bound reveals that $\pi^{*}$ degenerates toward the goal-agnostic prior when the value function weakly distinguishes among supported actions relative to $\alpha$, and the relabeled goal contains little additional information about the action when conditioned on the current state. Critically, both terms decay over the horizon. Notably, the first value-based term applies only to RL, whereas the second is applicable to both BC and RL.

\textbf{The diminishing action gap.} The action gap quantifies the difference in value between the optimal action and sub-optimal actions at a given state \citep{farahmand_action-gap_2011}. In goal-conditioned RL with sparse rewards and a discount factor $\gamma < 1$, the relative advantage of any single primitive action decays exponentially as the goal distance $d(s,g)$ increases. Consequently, the expected value span $\Delta_Q(S, G)$ shrinks, and the $Q$-values provide progressively less information about the optimal action distribution. As $d(s,g) \to \infty$, the first term approaches zero.

\begin{proposition}[Monotonic Decay of Mutual Information]
\label{prop:mi_decay}
Let dataset trajectories be generated by a Markovian behavior policy and transition kernel. For any time horizon (relabeling offset) $k' > k > 0$, given the current state $S_t$, the sequence $A_t \rightarrow S_{t+k} \rightarrow S_{t+k'}$ forms a conditional Markov chain. The mutual information between the action and future states monotonically decreases over time:
\begin{align*}
  I\left(A_t ; S_{t+k} \mid S_t\right) &- I\left(A_t ; S_{t+k'} \mid S_t\right) \\
  &= \mathbb{E}_{S_t, S_{t+k}, S_{t+k'}}\left[ D_{\mathrm{KL}}\left(p(A_t \mid S_t, S_{t+k}) \Vert{} p(A_t \mid S_t, S_{t+k'})\right) \right] \geq 0
\end{align*}

This inequality is strict when there exist at least two intermediate states, $s_{t+k}^{(1)}$ and $s_{t+k}^{(2)}$, that share a common future successor state $s_{t+k'}$ with non-zero probability, but yield distinct intermediate action posteriors $p(A_t \mid s_t, s_{t+k}^{(1)}) \neq p(A_t \mid s_t, s_{t+k}^{(2)})$.
\end{proposition}

\textbf{Non-increasing mutual information.} BC suffers a catastrophic performance degradation when trained on more distant relabeled goals (see GCBC and GCFBC in \Cref{fig:geom_horizon_oracle_success}), which is unsurprising as it does not benefit from the additional goal-related information imparted by a goal-conditioned value function. We formalize this as a value-independent reduction in mutual information $I_{\mathrm{HR}}(A ; G \mid S)$ as the relabeling horizon grows [\Cref{prop:mi_decay}]. An extended discussion is included in \Cref{appendix:proof}.

This bound only guarantees non-increasing information. Strictly decreasing information occurs in environments and datasets where distinct trajectories eventually \textit{merge} at a shared future state, which can occur in two common scenarios. First, when learning from a sub-optimal policy, subsequent behavior can ``overwrite'' the effects of earlier actions. An agent might make a mistake and later correct it, reaching the same future state as a trajectory that took an optimal path. Second, merging occurs when there are multiple valid ways to accomplish the same goal. For example, in a sequential pick-and-place task over three objects, there are up to six valid orderings that all lead to the same final configuration. In both scenarios, longer task and relabeling horizons provide more opportunities for these action-distinguishable trajectories to merge. Thus, temporally distant relabeled goals provide progressively less information about the immediate action.


\subsection{From action-goal mutual information to policy sensitivity}
\label{sec:link_to_Jac}

While our theoretical results characterize goal dependence as a property of the full distribution, it is difficult to measure this quantity in a policy parameterized by a neural network. One local measure of sensitivity is the input gradient (or Jacobian, for vector-valued outputs). As introduced in \Cref{subsec:input_jacobian_def}, we use the Jacobian of the policy action with respect to the state and goal as a first-order measure of policy sensitivity. A large goal Jacobian indicates that small changes in the goal induce large changes in the action, whereas a small goal Jacobian indicates local invariance to the goal.

It is important to acknowledge that this quantity is not an estimator of $I(A;G\mid S)$. Rather, it allows us to locally probe the action-goal dependence whose global behavior is captured by mutual information. For fixed-variance Gaussian policies, the squared Frobenius norm of the input Jacobian corresponds to the trace of the Fisher information matrix (FIM), and directly measures how rapidly the policy distribution changes with respect to changes in goal information [Box \ref{box:fisher}].

To investigate whether the goal Jacobian norm reflects the predicted decrease in action-goal MI, we measure the $\ell_1$ goal Jacobian norm proportion for policies trained on the goal distributions described in \Cref{subsec:empirical_oracle}. The geometric [Figures \ref{fig:geom_input_jac} \& \ref{fig:geom_input_jac_all}] and truncated uniform distributions [Figure \ref{fig:uniform_input_jac_all}] vary the proportion of temporally distant goals, which we expect to be less informative about the immediate action. The random goal mixture more directly reduces action-goal MI by introducing unrelated target states [Figure \ref{fig:all_random_actor_goals}]. Across BC and RL objectives, as well as Gaussian and flow-matching policy classes, we observe a consistent, \textit{global} shift in sensitivity away from the goal towards the current state across matched state-goal distances as the training horizon increases. As predicted, we see that goal sensitivity collapses for the BC policies, which do not utilize additional value-related information [Section \ref{sec:mi_theory}].

\begin{figure}[t]
  \begin{center}
  \includegraphics[width=1.0\linewidth]{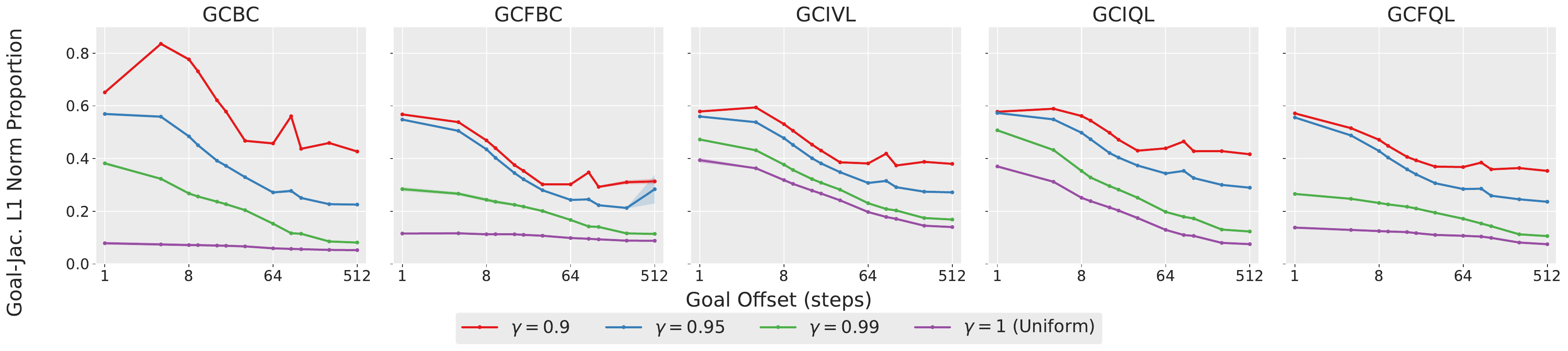}
  \caption{\textbf{Increasing training goal horizon \textit{globally} decreases policy goal sensitivity}. We measure the proportion of the input Jacobian norm corresponding to the goal at fixed state-goal offsets, plotted with $\log_2$ scaling in \texttt{cube-triple}. Results are averaged over the last three training checkpoints, three seeds, and 1024 randomly sampled state-goal pairs per horizon, with additional environments in Figure \ref{fig:geom_input_jac_all}.}
  \label{fig:geom_input_jac}
  \end{center}
\end{figure}

\begin{tcolorbox}[
  enhanced,
  breakable,
  float,
  floatplacement=t!,
  title=\textbf{Intuition}: Fisher Information Regularization,
  colframe=fav_color,
  colback=fav_color!8,
  coltitle=white,
  parbox=false,
  toprule=2pt,
  titlerule=1pt,
  leftrule=1pt,
  rightrule=1pt,
  bottomrule=1pt,
]
\label{box:fisher}

When $\pi$ is a Gaussian distribution with fixed variance $\sigma^2$, the state Jacobian $J_\pi(s)$ is the square root of the Fisher information matrix (FIM)
\begin{align*}
  \mathcal{I}(s)=\mathbb{E}_{\pi_\theta}\left[\nabla_s \log \pi_\theta(a \mid s, g) \cdot \nabla_s \log \pi_\theta(a \mid s, g)^T\right]
\end{align*}
with respect to the state. Consequently, the squared Frobenius norm of the state Jacobian, $\|J_s(s,g)\|_F^2$, when scaled by $\sigma^2$, corresponds exactly to the trace of the FIM and serves as a scalar measure of the local information the policy extracts from the state input $s$.

Viewed through this lens, our Sobolev bootstrapping objective [Equation \ref{eq:sobolev_objective}] acts as a dynamic form of \textbf{Fisher information regularization}. By penalizing deviations between the student's state Jacobian and that of the short-horizon teacher, we enforce the local information geometry of the long-horizon policy to be similar to that of the short-horizon policy.

This also provides a mechanical explanation for the findings in Section \ref{subsec:goal_cond} that regularizing the \textit{state} Jacobian also increases the norm of the \textit{goal} Jacobian. By regularizing local state dependence, Sobolev bootstrapping prevents the network from converging to a solution which exploits state-only optimization shortcuts and ignores goal conditioning inputs [Section \ref{sec:curse}].

\end{tcolorbox}

\section{Sobolev Bootstrapping}
The above analysis establishes a strong relationship between action-goal information and local policy sensitivity, where policies trained on longer-horizon goals exhibit reduced sensitivity to goal conditioning and a corresponding increase in sensitivity to the current state. This motivates us to investigate whether intervening on the policy input Jacobian can mitigate performance degradation. Direct $\ell_1$ and $\ell_2$ penalties on the state Jacobian yield modest gains on simpler tasks [\Cref{table:robustness}], providing evidence for a causal link between the policy sensitivity and downstream performance. However, these naive penalties fail to scale to more complex, multi-object environments.

We argue that a more principled Jacobian intervention should be \textit{context-dependent}, and propose \textbf{Sobolev bootstrapping}. Our approach uses a policy trained only on short-horizon goals as a first-order teacher for the full goal-conditioned ``student'' policy. Rather than globally suppressing state sensitivity via a uniform norm penalty, we regularize the long-horizon policy's state Jacobian towards that of the short-horizon policy, explicitly distilling the desirable sensitivity structure learned in a regime of high action-goal information. We additionally note that, for fixed-variance Gaussian policies, this objective can be interpreted as a form of Fisher information regularization [Box \ref{box:fisher}].

\subsection{Using short-horizon policies as teachers}
As the name suggests, Sobolev bootstrapping builds on a class of policy extraction algorithms which perform \textbf{policy bootstrapping} \citep{chane-sane_goal-conditioned_2021,zhou_flattening_2025}. These algorithms provide a natural foundation for our proposed method, since they already maintain a separate one-step policy trained only on short-horizon goals.

Due to space constraints, we provide only a brief intuition here and defer detailed descriptions of the value learning and policy extraction algorithms to \Cref{appendix:saw}.
Policy bootstrapping methods leverage the insight that, if an action $a$ is optimal for reaching a subgoal $s_g$ along the optimal path to goal $g$, then $a$ is also optimal for reaching $g$. Additionally, goal-conditioned policies are more likely to be accurate for short-horizon subgoals $s_g$, due to the decreasing ``signal-to-noise'' ratio of advantage estimates over the horizon \citep{park_hiql_2023}. Policy bootstrapping therefore uses $\pi(a\mid s, s_g)$ as a stronger, short-horizon supervisory signal than pure value-based policy extraction, by minimizing the Kullback-Leibler (KL) divergence between their action distributions
\begin{align*}
  \argmin_\theta D_\mathrm{KL}\left( \pi_\theta(a\mid s, g) \| \pi^\mathrm{sub}(a\mid s, s_g)\right).
\end{align*}

\textit{But how do we find good subgoals for a given goal}? Existing approaches such as Reinforcement Learning with Imagined Subgoals \citep[RIS]{chane-sane_goal-conditioned_2021} and Subgoal Advantage-Weighted Policy Bootstrapping \citep[SAW]{zhou_flattening_2025} differ primarily in how they answer this question. Similar to a bilevel hierarchical policy, RIS learns a generative model $\pi^h(s_g \mid s, g)$ to produce ``imagined'' subgoals. However, explicitly learning a hierarchical policy is known to suffer from compounding error and bottlenecks performance \citep{ahn_option-aware_2025}. On the other hand, SAW avoids these optimization difficulties by sampling only from in-trajectory waypoints $s_g \sim p^\mathcal D (s_{t+k} \mid s_t, g)$. It then re-weights the bootstrapping loss by $\exp(A(s,s_g, g))$, using a subgoal advantage estimate to prioritize high-quality waypoints [Equation \ref{eq:SAW_full}]. The advantage $A(s, s_g, g)$ is approximated by the learned value function as $V(s_g, g) - V(s,g)$, following \citet{park_hiql_2023}.

\subsection{Input Jacobian distillation}

While standard policy bootstrapping only distills the action distribution of the short-horizon teacher, our experiments in Section \ref{sec:curse} suggest that the input Jacobian of a short-horizon teacher may also provide useful supervisory signals. Drawing inspiration from Sobolev training \citep{czarnecki_sobolev_2017}, we augment the standard zeroth-order distillation with a first-order objective that matches the student's state Jacobian to that of the teacher policy $\pi^\mathrm{sub}(a\mid s, s_g)$:
\begin{align}
  \label{eq:sobolev_objective}
  \mathcal J_\mathrm{SB} = \mathbb E_{s,a,s_g, g} \left[\left \| \frac{\partial\pi_\theta( a\mid  s,  g)}{\partial  s} - \frac{\partial\pi^\mathrm{sub}( a\mid  s,  s_g)}{\partial s}\right \|_2^2\right].
\end{align}

Intuitively, the short-horizon teacher provides a prior not only for what the long-horizon policy should \textit{do}, but also what it should \textit{pay attention to} when deciding what to do. We match only the state Jacobian because the two policies receive different goal inputs: $\pi^\mathrm{sub}$ is conditioned on a nearby subgoal $s_g$ and $\pi$ on the final goal $g$. Interestingly, we find that regularizing the state Jacobian induces a corresponding change in the \textit{goal} Jacobian, which we discuss in more detail in Section \ref{subsec:goal_cond}.

\begin{table*}[t!]
  \caption{
  \textbf{Evaluating Sobolev bootstrapping on offline goal-conditioned RL tasks.}
  We compare our method's average (binary) success rate ($\%$) against the numbers reported in \citet{park_ogbench_2024} and \citet{zhou_flattening_2025} across the five test-time goals for each environment, averaged over the last three checkpoints and eight seeds $\pm$ standard deviations.
  Numbers within $5\%$ of the best value in the row are in \textbf{bold}.
  }
  \label{table:results}
  \centering
  \scalebox{0.66}
  {
  \begin{tabular}{llllllllll}
  \toprule
  \textbf{Environment} & \textbf{Dataset} & \textbf{GCBC} & \textbf{GCIVL} & \textbf{GCIQL} & \textbf{QRL} & \textbf{CRL} & \textbf{HIQL} & \textbf{SAW} & \textbf{S$^2$AW}\\
  \midrule
  \multirow[c]{3}{*}{\texttt{pointmaze}} & \texttt{pointmaze-medium-navigate-v0} & $9$ {\tiny $\pm 6$} & $63$ {\tiny $\pm 6$} & $53$ {\tiny $\pm 8$} & $82$ {\tiny $\pm 5$} & $29$ {\tiny $\pm 7$} & $79$ {\tiny $\pm 5$} & $\mathbf{97}$ {\tiny $\pm 2$} & $\mathbf{96}$ {\tiny $\pm 3$}\\
   & \texttt{pointmaze-large-navigate-v0} & $29$ {\tiny $\pm 6$} & $45$ {\tiny $\pm 5$} & $34$ {\tiny $\pm 3$} & $\mathbf{86}$ {\tiny $\pm 9$} & $39$ {\tiny $\pm 7$} & $58$ {\tiny $\pm 5$} & $\mathbf{85}$ {\tiny $\pm 10$} & $\mathbf{83}$ {\tiny $\pm 8$}\\
   & \texttt{pointmaze-giant-navigate-v0} & $1$ {\tiny $\pm 2$} & $0$ {\tiny $\pm 0$} & $0$ {\tiny $\pm 0$} & ${68}$ {\tiny $\pm 7$} & $27$ {\tiny $\pm 10$} & $46$ {\tiny $\pm 9$} & ${68}$ {\tiny $\pm 8$} & $\mathbf{73}$ {\tiny $\pm 9$}\\
  \cmidrule{1-10}
  \multirow[c]{3}{*}{\texttt{antmaze}} & \texttt{antmaze-medium-navigate-v0} & $29$ {\tiny $\pm 4$} & $72$ {\tiny $\pm 8$} & $71$ {\tiny $\pm 4$} & $88$ {\tiny $\pm 3$} & $\mathbf{95}$ {\tiny $\pm 1$} & $\mathbf{96}$ {\tiny $\pm 1$} & $\mathbf{97}$ {\tiny $\pm 1$} & $\mathbf{97}$ {\tiny $\pm 1$}\\
   & \texttt{antmaze-large-navigate-v0} & $24$ {\tiny $\pm 2$} & $16$ {\tiny $\pm 5$} & $34$ {\tiny $\pm 4$} & $75$ {\tiny $\pm 6$} & $83$ {\tiny $\pm 4$} & $\mathbf{91}$ {\tiny $\pm 2$} & $\mathbf{90}$ {\tiny $\pm 3$} & $\mathbf{93}$ {\tiny $\pm 2$}\\
   & \texttt{antmaze-giant-navigate-v0} & $0$ {\tiny $\pm 0$} & $0$ {\tiny $\pm 0$} & $0$ {\tiny $\pm 0$} & $14$ {\tiny $\pm 3$} & $16$ {\tiny $\pm 3$} & ${65}$ {\tiny $\pm 5$} & ${73}$ {\tiny $\pm 4$} & $\mathbf{87}$ {\tiny $\pm 3$}\\
  \cmidrule{1-10}
  \multirow[c]{3}{*}{\texttt{humanoidmaze}} & \texttt{humanoidmaze-medium-navigate-v0} & $8$ {\tiny $\pm 2$} & $24$ {\tiny $\pm 2$} & $27$ {\tiny $\pm 2$} & $21$ {\tiny $\pm 8$} & $60$ {\tiny $\pm 4$} & $\mathbf{89}$ {\tiny $\pm 2$} & $\mathbf{88}$ {\tiny $\pm 3$} & $\mathbf{90}$ {\tiny $\pm 2$}\\
   & \texttt{humanoidmaze-large-navigate-v0} & $1$ {\tiny $\pm 0$} & $2$ {\tiny $\pm 1$} & $2$ {\tiny $\pm 1$} & $5$ {\tiny $\pm 1$} & $24$ {\tiny $\pm 4$} & ${49}$ {\tiny $\pm 4$} & ${46}$ {\tiny $\pm 4$} & $\mathbf{57}$ {\tiny $\pm 5$}\\
   & \texttt{humanoidmaze-giant-navigate-v0} & $0$ {\tiny $\pm 0$} & $0$ {\tiny $\pm 0$} & $0$ {\tiny $\pm 0$} & $1$ {\tiny $\pm 0$} & $3$ {\tiny $\pm 2$} & ${12}$ {\tiny $\pm 4$} & ${35}$ {\tiny $\pm 4$} & $\mathbf{39}$ {\tiny $\pm 5$}\\
  \cmidrule{1-10}
   \multirow[c]{4}{*}{\texttt{cube}} & \texttt{cube-single-play-v0} & $6$ {\tiny $\pm 2$} & $53$ {\tiny $\pm 4$} & ${68}$ {\tiny $\pm 6$} & $5$ {\tiny $\pm 1$} & $19$ {\tiny $\pm 2$} & $44$ {\tiny $\pm 9$} & ${72}$ {\tiny $\pm 5$} & $\mathbf{86}$ {\tiny $\pm 5$}\\
   &  \texttt{cube-double-play-v0} & $1$ {\tiny $\pm 1$} & $36$ {\tiny $\pm 3$} & ${40}$ {\tiny $\pm 5$} & $1$ {\tiny $\pm 0$} & $10$ {\tiny $\pm 2$} & $6$ {\tiny $\pm 2$} & ${40}$ {\tiny $\pm 7$} & $\mathbf{57}$ {\tiny $\pm 6$}\\
   &  \texttt{cube-triple-play-v0} & $1$ {\tiny $\pm 1$} & $1$ {\tiny $\pm 0$} & $3$ {\tiny $\pm 1$} & $0$ {\tiny $\pm 0$} & ${4}$ {\tiny $\pm 1$} & $3$ {\tiny $\pm 1$} & ${4}$ {\tiny $\pm 2$} & $\mathbf{30}$ {\tiny $\pm 5$}
   \\
   &  \texttt{cube-quadruple-play-v0} & $0$ {\tiny $\pm 0$} & $0$ {\tiny $\pm 0$} & $0$ {\tiny $\pm 0$} & $0$ {\tiny $\pm 0$} & $0$ {\tiny $\pm 0$} & $0$ {\tiny $\pm 0$} & $0$ {\tiny $\pm 0$} & $\mathbf{3}$ {\tiny $\pm 2$}
   \\
\cmidrule{1-10}
  \multirow[c]{1}{*}{\texttt{scene}} & \texttt{scene-play-v0} & $5$ {\tiny $\pm 1$} & $42$ {\tiny $\pm 4$} & ${51}$ {\tiny $\pm 4$} & $5$ {\tiny $\pm 1$} & $19$ {\tiny $\pm 2$} & $38$ {\tiny $\pm 3$} & ${63}$ {\tiny $\pm 6$} & $\mathbf{74}$ {\tiny $\pm 6$}\\
\cmidrule{1-10}
   \multirow[c]{2}{*}{\texttt{puzzle}}
   & \texttt{puzzle-3x3-play-v0} & $2$ {\tiny $\pm 0$} & $6$ {\tiny $\pm 1$} & $\mathbf{95}$ {\tiny $\pm 1$} & $1$ {\tiny $\pm 0$} & $3$ {\tiny $\pm 1$} & $12$ {\tiny $\pm 2$} & ${9}$ {\tiny $\pm 3$} & ${15}$ {\tiny $\pm 4$}\\
   &  \texttt{puzzle-4x4-play-v0} & $0$ {\tiny $\pm 0$} & $13$ {\tiny $\pm 2$} & ${26}$ {\tiny $\pm 3$} & $0$ {\tiny $\pm 0$} & $0$ {\tiny $\pm 0$} & $7$ {\tiny $\pm 2$} & ${26}$ {\tiny $\pm 3$} & $\mathbf{41}$ {\tiny $\pm 3$}\\
  \bottomrule
  \end{tabular}
}
\end{table*}

\subsection{A practical policy extraction algorithm}
Sobolev bootstrapping requires access to differentiable \textit{subgoal}-conditioned ``teacher'' policies, making existing policy bootstrapping algorithms a natural foundation for our approach. For our primary results in Table \ref{table:results}, we instantiate Sobolev bootstrapping on top of SAW, which we denote as \textbf{S${^2}$AW} [Algorithm \ref{alg:s2aw}], due to its simplicity and strong performance. We modify the objective in Equation \ref{eq:sobolev_objective} to add the same subgoal advantage-weighted coefficient as in SAW [\Cref{appendix:saw}], resulting in the objective
\begin{align}
  \label{eq:sobolev_saw_obj}
  \mathcal J_\mathrm{SB}= \mathbb E_{s,a,s_g, g} \left[e^{A(s,s_g, g)}\left \| J_{\pi_\theta}(s) - J_{\pi^\mathrm{sub}}(s)\right \|_2^2\right].
\end{align}

While we focus our main analysis on S${^2}$AW, we also demonstrate the effectiveness of Sobolev bootstrapping on RIS, achieving substantial gains over the baseline algorithm [\Cref{appendix:ris_results}].

\section{Experiments}

We evaluate Sobolev bootstrapping on the OGBench suite \citep{park_ogbench_2024}, which spans long-horizon locomotion and combinatorial manipulation tasks. Our experiments address two central questions: (1) does Sobolev bootstrapping improve goal-reaching performance, and (2) does it reverse the horizon-induced shift in policy sensitivity identified in our previous experiments? Interestingly, although Sobolev bootstrapping only distills the \textit{state} Jacobian, we find that it also restores goal-conditioning sensitivity [Section \ref{subsec:goal_cond}].

\subsection{Sobolev bootstrapping boosts long-horizon performance}
Adding Sobolev bootstrapping to SAW (S$^2$AW) improves performance across most environments [Table \ref{table:results}], with the largest gains in combinatorial manipulation tasks such as the \texttt{cube} and \texttt{puzzle} environments. In particular, S$^2$AW improves from $40\%$ to $57\%$ on \texttt{cube-double}, $4\%$ to $30\%$ on \texttt{cube-triple}, and $26\%$ to $41\%$ on \texttt{puzzle-4x4}. We also observe substantial improvements in long-horizon locomotion, although these results may be partially attributed to architectural differences [Appendix \ref{appendix:arch_ablations}]. Additional horizon reduction baselines are provided in \Cref{appendix:additional_baselines}.

The improvements in the discrete \texttt{puzzle} environments demonstrate that first-order distillation can be practically useful even when the true decision boundary is non-smooth, since the short-horizon teacher policy provides a (piecewise) smooth, differentiable approximation of these hard boundaries.

\subsection{Effects on state sensitivity}
\label{subsec:sobolev_sensitivity}

\begin{figure}[t!]
  \begin{center}
  \includegraphics[width=1.0\linewidth]{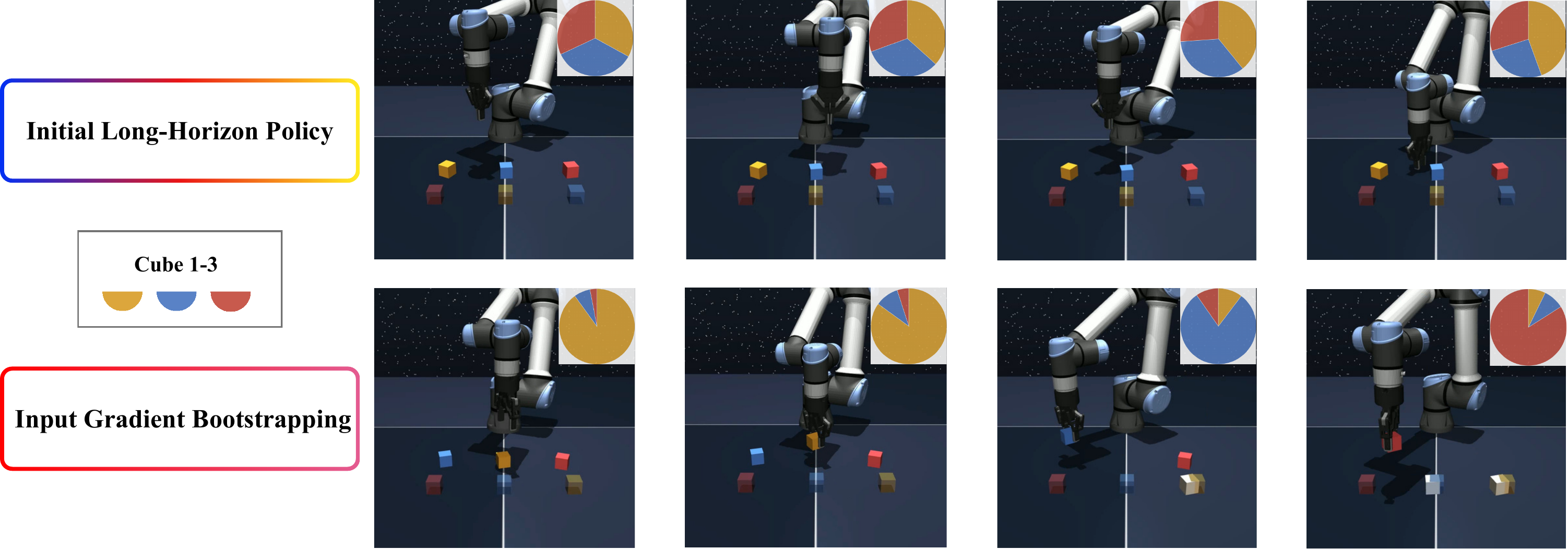}
  \caption{\textbf{Sobolev bootstrapping exhibits more interpretable sensitivity structure}. We plot the proportion of the total input Jacobian norm corresponding to each cube during a triple pick-and-place \textit{on-policy} rollout for SAW without (top) and with (bottom) Sobolev bootstrapping. For fairness, we also analyze a scripted optimal trajectory in \Cref{fig:saw_vs_s2aw_saliency_map}.}
  \label{fig:pie_chart}
  \end{center}
\end{figure}

To examine how Sobolev bootstrapping qualitatively changes policy sensitivity to state information, we decompose the state Jacobian norm by object over the course of a \texttt{cube-triple} rollout [Figure \ref{fig:pie_chart}], allowing us to visualize where the policy's ``attention'' is allocated. Vanilla SAW exhibits diffuse sensitivity across all cubes and immediately fails at the first task stage, whereas S$^2$AW dynamically allocates sensitivity to the focal cube at each stage and successfully completes the full triple pick-and-place task. To control for differences in policy competence, we repeat the analysis on an independent, scripted optimal trajectory in \Cref{fig:saw_vs_s2aw_saliency_map} and observe the same pattern.

We hypothesize that the observed shift in sensitivity away from goal information and towards state inputs may also result in policies that overfit to task-irrelevant ``distractors.'' To test this, we evaluate policies trained in multi-cube environments on single pick-and-place tasks. The first evaluation task in \texttt{cube-double} and \texttt{triple} already tests this capability, and we add a similar task to \texttt{cube-quadruple}, which we refer to as Task 6. While naive $\ell_1$ and $\ell_2$ Jacobian norm penalties produce small improvements, context-dependent Sobolev bootstrapping significantly improves performance even on \texttt{cube-triple} and \texttt{quadruple} [Table \ref{table:robustness}].

\begin{figure}[t]
  \begin{center}
  \includegraphics[width=1.0\linewidth]{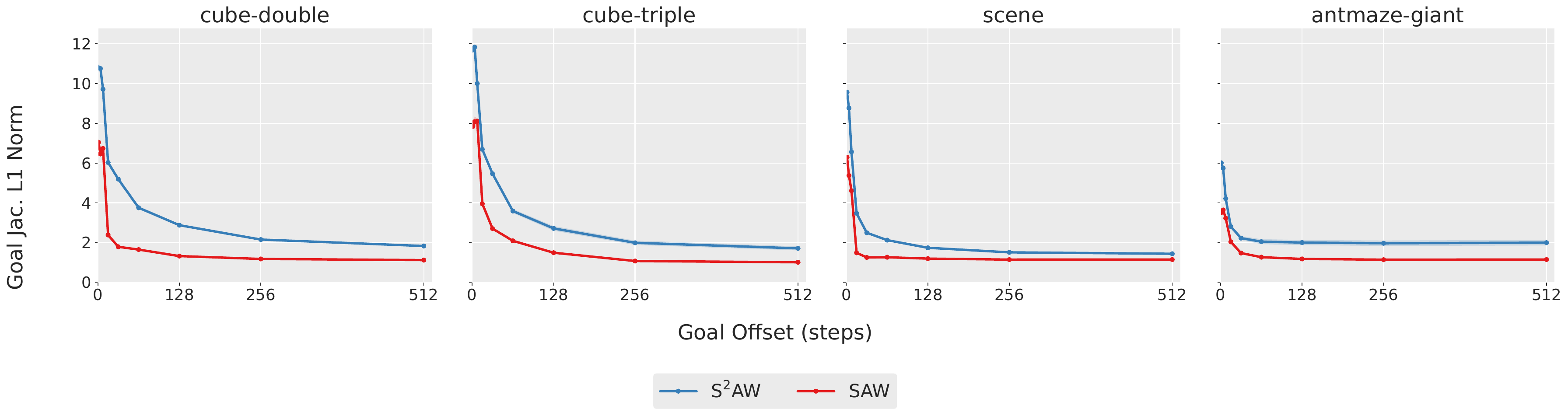}
  \caption{\textbf{Sobolev bootstrapping recovers policy goal sensitivity}. We compare the \textit{absolute} goal Jacobian norm between S$\mathbf{^2}$AW and SAW, averaging over batches of goals sampled at fixed intervals from the current state up to the maximum horizon. Results are averaged across four seeds and 1024 state-goal pairs for each horizon, with shaded standard deviations.}
  \label{fig:saw_eval_input_grad}
  \end{center}
\end{figure}

\begin{table}[t]
  \caption{
  \textbf{Sobolev bootstrapping improves single pick-and-place performance in the presence of distractor objects.}
  We evaluate policies trained in multi-cube environments on single pick-and-place tasks. While naive $\ell_1$ or $\ell_2$ Jacobian regularization [Appendix \ref{appendix:l1_l2_jacobian_reg}] yields modest improvements on simpler tasks, Sobolev bootstrapping matches or outperforms all baselines. Other settings are identical to \Cref{table:results}.
  }
  \label{table:robustness}
  \centering
  \scalebox{0.7}
  {
  \begin{tabular}{lllllll}
  \toprule
  \textbf{Environment} & \textbf{GCIQL} & \textbf{GCIQL$_{\ell_1}$} & \textbf{SAW} & \textbf{SAW$_{\ell_1}$} & \textbf{SAW$_{\ell_2}$} & \textbf{S$^2$AW}\\
  \midrule
  \texttt{cube-single} & $68$ {\tiny $\pm 6$} & $74$ {\tiny $\pm 4$} & $76$ {\tiny $\pm 8$} & $\mathbf{85}$ {\tiny $\pm 6$} & $\mathbf{85}$ {\tiny $\pm 4$} & $\mathbf{86}$ {\tiny $\pm 5$} \\
  \texttt{cube-double-task1} & $74$ {\tiny $\pm 8$} & $76$ {\tiny $\pm 8$} & $64$ {\tiny $\pm 9$} & $\mathbf{81}$ {\tiny $\pm 9$} & ${65}$ {\tiny $\pm 14$} & $\mathbf{78}$ {\tiny $\pm 12$}\\
  \texttt{cube-triple-task1} & $13$ {\tiny $\pm 3$} & $16$ {\tiny $\pm 10$} & $18$ {\tiny $\pm 11$} & $24$ {\tiny $\pm 9$} & $17$ {\tiny $\pm 12$} & $\mathbf{74}$ {\tiny $\pm 9$}\\
  \texttt{cube-quadruple-task6} & $0$ {\tiny $\pm 1$} & $0$ {\tiny $\pm 1$} & $2$ {\tiny $\pm 2$} & $2$ {\tiny $\pm 2$} & $1$ {\tiny $\pm 1$} & $\mathbf{35}$ {\tiny $\pm 14$}\\
  \bottomrule
  \end{tabular}
}
\end{table}

\subsection{Effects on goal conditioning}
\label{subsec:goal_cond}

Although Sobolev bootstrapping explicitly matches only the \textit{state} Jacobian, we find that it induces an increase in the policy's \textit{absolute} goal Jacobian norm [Figure \ref{fig:saw_eval_input_grad}]. This suggests that the policy has converged to a different, more goal-sensitive solution instead of simply smoothing out state sensitivity. We note that a decrease in sensitivity to distant goals is unsurprising---for instance, perturbing the target placement in a pick-and-place task should not affect the policy before grasping the cube. However, standard long-horizon policies exhibit \textit{global} insensitivity across all goal distances, including \textit{nearby} goals [Figure \ref{fig:saw_eval_input_grad}], and it is this pathology that Sobolev bootstrapping addresses.

\section{Discussion}
\label{sec:discussion}
In this work, we identify an \textit{informational} curse of horizon in goal-conditioned policy learning. By decoupling the training goal horizon from the subgoal distance presented at test time, we show that policies trained on distant relabeled goals can become worse at reaching nearby goals. We explain this effect through the decay of action-goal mutual information under hindsight relabeling: as training goals become more distant, they provide progressively less information about the immediate action. This is accompanied by a global shift in sensitivity away from the goal and toward the current state in neural network policies, across all goal distances. Although input Jacobians are only a local probe of this dependence, explicitly weakening action-goal information produces a corresponding collapse in goal sensitivity. We then propose Sobolev bootstrapping, which exploits this observation by distilling the sensitivity structure of short-horizon policies into long-horizon policies and achieves large gains in combinatorial manipulation tasks. Overall, our results suggest that the goal-relabeling horizon should be treated as an important design choice, rather than a mere consequence of task horizon.

Our work does have several limitations. The optimal goal-horizon distribution depends strongly on the policy objective, prompting further study on how different policy extraction methods utilize goal information. Additionally, input Jacobians capture only local sensitivity and are not direct estimators of mutual information, and the Sobolev bootstrapping algorithm derived from those insights introduces significant additional computation. These limitations motivate future work on information-aware relabeling and policy-learning objectives that preserve goal dependence without explicit Jacobian distillation.

\vfill
\pagebreak

\section*{Acknowledgments}
This work was supported by NIH R01NS121097 awarded to JCK.

\bibliography{iclr2027_conference}
\bibliographystyle{iclr2027_conference}

\vfill
\pagebreak

\appendix
\section{Related Work}

\subsection{Offline goal-conditioned policy learning}

The objective of offline goal-conditioned policy learning is to learn a policy which can reach any state from any other state from a fixed dataset of offline trajectories. This can be accomplished with goal-conditioned behavioral cloning (GCBC) that relabels state-action pairs with future trajectory states as goals; or goal-conditioned reinforcement learning (GCRL), which makes use of learned goal-conditioned value functions. It combines the challenges of conservative offline RL with the sparse rewards and recursive structure of goal-reaching tasks \citep{kaelbling_learning_1993}. Standard algorithms typically use behavior-regularized actor-critic methods or learn an \textit{implicit} value function with in-sample value iteration \citep{kostrikov_offline_2021}, then extract a policy using weighted behavioral cloning \citep{peng_advantage-weighted_2019} or reparameterized policy gradient with a behavioral cloning penalty \citep{fujimoto_minimalist_2021}. The recursive structure of goal-conditioned RL also admits a variety of non-traditional value learning algorithms based on successor measures \citep{eysenbach_contrastive_2022} and quasimetric structure \citep{wang_optimal_2023,park_transitive_2026}.

\subsection{The curse of horizon}
\label{subsec:curse}
The challenges introduced by horizon in sequential decision-making tasks are well known. In addition to the error bounds of \citet{ross_efficient_2010} in imitation learning, reinforcement learning introduces a whole host of additional horizon-related difficulties. The effective task horizon, set by the discount factor $\gamma$, is known to control the complexity of the policy class \citep{jiang_dependence_2015,jiang_structural_2016}, and lower discount factors can provide pessimistic \citep{hu_role_2022} and regularizing effects \citep{amit_discount_2020}. However, artificially decreasing the discount factor is not feasible when the true task horizon is long and reward signals are sparse, as in goal-reaching tasks.

Increasing the discount factor in off-policy RL leads to a host of additional problems, typically attributed to bias accumulation in value learning \citep{park_horizon_2025} and difficulty in advantage estimation for primitive actions in policy extraction. To mitigate these difficulties, recent approaches have employed various techniques to reduce the effective horizon relative to decision-making frequency, both for value learning \citep{ahn_option-aware_2025,zheng_scaling_2025} and policy extraction \citep{li_reinforcement_2025,li_decoupled_2026}. Our work specifically focuses on the effects of horizon on \textit{goal-conditioned} policy learning, where we identify a particular \textit{informational} curse of horizon: as the horizon increases, the informative signal between the policy and the goal diminishes, causing the policy to overfit to spurious state features at the expense of goal conditioning.

\subsection{Steerable robot policies}
Effective goal conditioning is increasingly critical for modern robotic systems, including subgoal-conditioned policies \citep{zhou2025act2goal,intelligence2026pi07steerablegeneralistrobotic} and inverse-kinematic-style World Action Models \citep{du2023learninguniversalpoliciestextguided,li2026causalworldmodelingrobot}. Recent works demonstrate the significance of effectively utilizing goal information: \citet{mishra2026understandingmitigatingvideoactiongeneralization} links video-action model (VAM) performance to a model's reliance on the generated future frames, as measured through cross-attention ratios. \citet{korniak2026stepstimelearningrepresentations} shows that action chunking increases the mutual information between the state-action pair and the goal, which facilitates critic learning in goal-conditioned RL.


\newpage

\section{Extended Proofs}
\label{appendix:proof}
Below, we use uppercase letters (e.g., $S_t$) to denote random variables, and lower-case letters (e.g., $s_t$) to denote their instantiations.

\begin{lemma}
Let $P$ be a base probability distribution and let $f(x)$ be a bounded scalar function such that $f(x) \in [a, b]$ for all $x$ in the support of $P$. Define the exponentially reweighted distribution $P_f$ as $P_f(x) = \frac{P(x) e^{f(x)}}{Z}$, where $Z = \mathbb{E}_{X \sim P}[e^{f(X)}]$ is the partition function. Then, the Kullback-Leibler divergence is bounded by:
$$
D_{\mathrm{KL}}(P_f \Vert P) \leq \frac{(b - a)^2}{8}
$$
\end{lemma}

\begin{proof}
The KL divergence between a base distribution $P$ and its exponentially reweighted version $P_f \propto P e^f$ can be expressed in terms of the log-partition function $\psi(\lambda) = \log \mathbb{E}_{X \sim P}[e^{\lambda f(X)}]$, where $\lambda \in [0, 1]$ is a scalar parameter that defines the intermediate distributions between the base distribution and the exponentially reweighted target. By taking a second-order Taylor expansion of $\psi(\lambda)$, we can rewrite this KL divergence as the remainder $\frac{1}{2} \psi''(c)$ for some $c \in (0,1)$. Because $f(X)$ is bounded in $[a, b]$, we can apply Hoeffding's Lemma (see Lemma 2.2 in \cite{boucheron-hal-00942704}) to bound this second derivative, yielding a maximum KL divergence of $\frac{(b - a)^2}{8}$.
\end{proof}

\begin{theorem}[Divergence Bound from the Goal-Agnostic Prior]
Let $\pi^{*}(a \mid s, g) \propto p_{\mathrm{HR}}(a \mid s, g) e^{Q(s, a, g) / \alpha}$ be the target distribution induced by the KL-regularized policy update, where $\alpha$ is a temperature parameter and $p_{\mathrm{HR}}$ is the distribution of tuples $(S, A, G)=\left(S_t, A_t, S_{t+K}\right)$ induced by the dataset trajectory distribution and a hindsight relabeling rule $K \sim q_{\mathrm{HR}}$. Let $\beta(a \mid s)$ be the goal-agnostic behavior policy induced by the dataset, and let the critic value span be $\Delta_Q(s, g) = \sup_{a} Q(s, a, g) - \inf_{a} Q(s, a, g)$ over the support of $p_{\mathrm{HR}}$. Then, the expected squared total variation distance between the two policies is bounded by:
$$
\mathbb{E}_{S, G}\left[D_{\mathrm{TV}}\left(\pi^{*}(\cdot \mid S, G), \beta(\cdot \mid S)\right)^2\right] \leq \frac{\mathbb{E}_{S, G}\left[\Delta_Q(S, G)^2\right]}{8 \alpha^2} + I_{\mathrm{HR}}(A ; G \mid S)
$$
\end{theorem}

\begin{proof}
We bound the total variation distance (TVD) between the target policy $\pi^*(\cdot \mid s, g)$ and the behavioral policy $\beta(\cdot \mid s)$ by introducing the hindsight-relabeled distribution $p_{\mathrm{HR}}(\cdot \mid s, g)$ as an intermediate distribution. We denote $\pi^*(\cdot \mid s, g)$ as $\pi^*$ below for brevity. The same notation applies to $\beta(\cdot \mid s)$ and $p_{\mathrm{HR}}(\cdot \mid s, g)$.

Because TVD is a metric, it satisfies the triangle inequality. For any given state $s$ and goal $g$:
\begin{equation*}
D_{\mathrm{TV}}(\pi^*, \beta) \leq D_{\mathrm{TV}}(\pi^*, p_{\mathrm{HR}}) + D_{\mathrm{TV}}(p_{\mathrm{HR}}, \beta).
\end{equation*}

After squaring both sides and applying the inequality $(x + y)^2 \leq 2x^2 + 2y^2$, we have
\begin{equation*}
D_{\mathrm{TV}}(\pi^*, \beta)^2 \leq 2 D_{\mathrm{TV}}(\pi^*, p_{\mathrm{HR}})^2 + 2 D_{\mathrm{TV}}(p_{\mathrm{HR}}, \beta)^2.
\end{equation*}

Pinsker's inequality, $D_{\mathrm{TV}}(P, Q)^2 \leq \frac{1}{2} D_{\mathrm{KL}}(P \Vert Q)$, applied to $2 D_{\mathrm{TV}}(\pi^*, p_{\mathrm{HR}})^2$, gives us
\begin{equation*}
2 D_{\mathrm{TV}}(\pi^*, p_{\mathrm{HR}})^2 \leq D_{\mathrm{KL}}(\pi^* \Vert p_{\mathrm{HR}}).
\end{equation*}

Because $\pi^*$ is an exponentially reweighted version of $p_{\mathrm{HR}}$, we can apply Lemma 1 to $D_{\mathrm{KL}}(\pi^* \Vert p_{\mathrm{HR}})$ with $f(a) = Q(s, a, g) / \alpha$, yielding
\begin{equation*}
D_{\mathrm{KL}}(\pi^* \Vert p_{\mathrm{HR}}) \leq \frac{\Delta_Q(s, g)^2}{8 \alpha^2}.
\end{equation*}

Now we have shown
\begin{equation} \label{eq:intermediate_tvd}
D_{\mathrm{TV}}(\pi^*, \beta)^2 \leq \frac{\Delta_Q(s, g)^2}{8 \alpha^2} + 2 D_{\mathrm{TV}}(p_{\mathrm{HR}}, \beta)^2.
\end{equation}

Next, we apply Pinsker's inequality again to $2 D_{\mathrm{TV}}(p_{\mathrm{HR}}, \beta)^2$, leading to
\begin{equation*}
2 D_{\mathrm{TV}}(p_{\mathrm{HR}}, \beta)^2 \leq D_{\mathrm{KL}}(p_{\mathrm{HR}}(\cdot \mid s, g) \Vert \beta(\cdot \mid s)).
\end{equation*}

Note that hindsight relabeling does not change the marginal state distribution, implying $\beta(a \mid s) = p_{\mathrm{HR}}(a \mid s)$. Then taking the expectation of this KL divergence over $S$ and $G$ yields the conditional mutual information between actions and goals given the state:
\begin{equation} \label{eq:mi_expectation}
\mathbb{E}_{S, G}\left[ D_{\mathrm{KL}}(p_{\mathrm{HR}}(\cdot \mid S, G) \Vert p_{\mathrm{HR}}(\cdot \mid S)) \right] = I_{\mathrm{HR}}(A ; G \mid S)
\end{equation}

Taking the expectation of \eqref{eq:intermediate_tvd} over S and G and combining it with \eqref{eq:mi_expectation} concludes the proof.
\end{proof}

\begin{proposition}[Monotonic Decay of Mutual Information]
\label{prop:mi_decay_appendix}
Let dataset trajectories be generated by a Markovian behavior policy and transition kernel. For any time horizon (relabeling offset) $k' > k > 0$, given the current state $S_t$, the sequence $A_t \rightarrow S_{t+k} \rightarrow S_{t+k'}$ forms a conditional Markov chain. The mutual information between the action and future states monotonically decreases over time:
\begin{align*}
  I\left(A_t ; S_{t+k} \mid S_t\right) &- I\left(A_t ; S_{t+k'} \mid S_t\right) \\
  &= \mathbb{E}_{S_t, S_{t+k}, S_{t+k'}}\left[ D_{\mathrm{KL}}\left(p(A_t \mid S_t, S_{t+k}) \Vert{} p(A_t \mid S_t, S_{t+k'})\right) \right] \geq 0
\end{align*}

This inequality is strict when there exist at least two intermediate states, $s_{t+k}^{(1)}$ and $s_{t+k}^{(2)}$, that share a common future successor state $s_{t+k'}$ with non-zero probability, but yield distinct intermediate action posteriors $p(A_t \mid s_t, s_{t+k}^{(1)}) \neq p(A_t \mid s_t, s_{t+k}^{(2)})$.
\end{proposition}

\begin{proof}
Expand $I(A_t ; S_{t+k}, S_{t+k'} \mid S_t)$ with mutual information chain rule in both ways and leverage $I\left(A_t ; S_{t+k^{\prime}} \mid S_t, S_{t+k}\right) = 0$ implied by the Markov chain, we have
\begin{equation*}
I\left(A_t ; S_{t+k} \mid S_t\right) - I\left(A_t ; S_{t+k^{\prime}} \mid S_t\right) = I\left(A_t ; S_{t+k} \mid S_t, S_{t+k^{\prime}}\right) \geq 0.
\end{equation*}

Next, we rewrite $I\left(A_t ; S_{t+k} \mid S_t, S_{t+k^{\prime}}\right)$ as an expected KL divergence
\begin{equation*}
\begin{split}
& I(A_t ; S_{t+k} \mid S_t, S_{t+k'}) \\
& \quad = \mathbb{E}_{S_t, S_{t+k'}} \left[ D_{\mathrm{KL}}\left(p(A_t, S_{t+k} \mid S_t, S_{t+k'}) \Vert p(A_t \mid S_t, S_{t+k'})p(S_{t+k} \mid S_t, S_{t+k'})\right) \right].
\end{split}
\end{equation*}

Since $p(A_t, S_{t+k} \mid S_t, S_{t+k'}) = p(S_{t+k} \mid S_t, S_{t+k'}) p(A_t \mid S_t, S_{t+k})$, we can simplify the KL term
\begin{equation*}
\begin{split}
& I(A_t ; S_{t+k} \mid S_t, S_{t+k'}) \\
& \quad = \mathbb{E}_{S_t, S_{t+k^{\prime}}, (S_{t+k} \mid S_t, S_{t+k^{\prime}})}\left[ D_{\mathrm{KL}}\left(p\left(A_t \mid S_t, S_{t+k}\right) \Vert p\left(A_t \mid S_t, S_{t+k^{\prime}}\right)\right) \right].
\end{split}
\end{equation*}

To see when this inequality is strict (i.e., when mutual information strictly decreases with an increase in relabeling horizon), we compare the intermediate action posterior $p(A_t \mid S_t, S_{t+k})$ and the future action posterior $p(A_t \mid S_t, S_{t+k'})$ within the KL term. We instantiate some start state $s_t$ and future state $s_{t+k'}$ to focus on the KL term inside the expectation. By the law of total probability, the future posterior is the marginalization of the intermediate posteriors over all possible intermediate states:
\begin{equation*}
p(A_t \mid s_t, s_{t+k'}) = \mathbb{E}_{S_{t+k} \mid s_t, s_{t+k'}} [p(A_t \mid s_t, S_{t+k})].
\end{equation*}

Now, to make the KL term non-zero, we need $\mathbb{E}_{S_{t+k} \mid s_t, s_{t+k'}} [p(A_t \mid s_t, S_{t+k})] \neq p(A_t \mid s_t, s_{t+k})$ for some instantiation $s_{t+k}$. In other words, the intermediate posterior $p(A_t \mid s_t, S_{t+k})$ cannot be almost-surely constant with respect to $S_{t+k}$.

For the intermediate posterior $p(A_t \mid s_t, S_{t+k})$ to vary depending on $S_{t+k}$, we need at least two instantiations, $s_{t+k}^{(1)}$ and $s_{t+k}^{(2)}$, that simultaneously satisfy the two conditions below:

\textbf{Merging paths:} both intermediate states must be viable nodes on the path from the start state $s_t$ to the future state $s_{t+k'}$. In other words, they must both occur with non-zero probability given the endpoints:
\begin{equation*}
p(s_{t+k}^{(1)} \mid s_t, s_{t+k'}) > 0 \quad \text{and} \quad p(s_{t+k}^{(2)} \mid s_t, s_{t+k'}) > 0
\end{equation*}

\textbf{Different actions:} knowing which of these intermediate states the agent passed through must change our belief about the action taken at time $t$:
\begin{equation*}
p(A_t \mid s_t, s_{t+k}^{(1)}) \neq p(A_t \mid s_t, s_{t+k}^{(2)})
\end{equation*}


The two conditions are satisfied and strict loss of information occurs in environments and datasets where distinct trajectories eventually ``merge'' at a shared future state. Intuitively, this can occur in two common scenarios. First, when learning from a sub-optimal policy, subsequent behavior can \textit{overwrite} the effects of earlier actions. An agent might make a mistake and later correct it, reaching the same future state as a trajectory that took an optimal path. Second, merging occurs when there are multiple valid ways to accomplish the same goal. For example, in a task requiring an agent to pick and place three cubes, there are six valid orderings that all lead to the same final state. In both scenarios, longer task and relabeling horizons provide more opportunities for these action-distinguishable trajectories to merge. Thus, temporally distant relabeled goals become progressively less informative about the original action.

Our analysis extends to distributions over relabeling offsets when $K \sim q_{\text{HR}}$ is sampled independently of the dataset trajectory. The chain rule for mutual information gives $I_{\text{HR}}(A; G \mid S) = \mathbb{E}_{K \sim q_{\text{HR}}}[I(A_t; S_{t+K} \mid S_t)] - I_{\text{HR}}(A; K \mid S, G) \leq \mathbb{E}_{K \sim q_{\text{HR}}}[I(A_t; S_{t+K} \mid S_t)]$. Therefore, by \Cref{prop:mi_decay_appendix}, shifting $q_{\text{HR}}$ towards longer offsets decreases this upper bound on the action-goal mutual information.

\end{proof}

\newpage

\section{Algorithm Pseudocode}

\begin{algorithm*}
  \caption{Sobolev Bootstrapping}
  \label{alg:s2aw}
  \begin{algorithmic}[1]
    \STATE \textbf{Input}: offline dataset $\mathcal D$, goal distribution $p(g)$.
    \STATE Initialize value function $V_\beta$, target subpolicy $\pi^{\mathrm{sub}}_\phi$, and policy $\pi_\theta$.
    \WHILE {not converged}
    \STATE Train value function: $\beta \leftarrow \beta - \lambda \nabla_\beta \mathcal L_{\mathrm{GCIVL}}(\beta)$ with $(s_t, s_{t+1}) \sim p^\mathcal D, g\sim p(g)$ [Equation \ref{eq:IVL}]
    \ENDWHILE
    \WHILE {not converged}
    \STATE Train target subpolicy: $\phi \leftarrow \phi + \lambda \nabla_\phi \mathcal J_{\mathrm{AWR}}(\phi)$ with $(s_t, a, s_g) \sim p^\mathcal D$ [Equation \ref{eq:low_policy_AWR}]
    \ENDWHILE
    \WHILE {not converged}
    \STATE Train policy: $\theta \leftarrow \theta - \lambda \nabla_\theta \left( \mathcal J_\mathrm{SB}(\theta) - \mathcal J_\mathrm{SAW}(\theta) \right)$ with $(s_t, a, s_g) \sim p^\mathcal D, g\sim p(g)$ [Equation \ref{eq:sobolev_saw_obj} \& \ref{eq:SAW_full}]
    \ENDWHILE
  \end{algorithmic}
\end{algorithm*}

\section{Algorithmic Details}
\label{appendix:saw}
In this section, we provide full details of the value and policy learning objectives for our Sobolev bootstrapping algorithm, which is instantiated on top of Subgoal Advantage-Weighted Policy Bootstrapping \citep[SAW]{zhou_flattening_2025}.

\subsection{Goal-conditioned implicit value learning (GCIVL)}

To learn a value function $V(s,g)$ that can operate both on single- and multi-step transitions, we use a goal-conditioned variant of implicit Q-learning (IQL) \citep{kostrikov_offline_2021}, which approximates the $\max$ operator with an expectile regression on in-sample state-action pairs:
\begin{align}
    \label{eq:IQL}
    &\mathcal{L}_V(\beta_V)=\mathbb{E}_{(s,a,g)\sim D}\left[L_2^{\tau}(Q_{\bar\beta_Q}(s,a,g)-V_{\beta_V}(s,g))\right] \\
    &\mathcal{L}_Q(\beta_Q)=\mathbb{E}_{(s,a,s',g)\sim D} \left[(r(s,g)+\gamma V_{{\beta}_V}(s',g)-Q_{\beta_Q}(s,a,g))^2 \right],
\end{align}

where $L_2^{\tau}(x)=|\tau-\mathbbm{1}(x < 0)| x^2$ is the expectile loss and $\tau \in [0.5,1)$ is the expectile parameter.

Following previous works \citep{park_hiql_2023}, we adopt an action-free variant of this objective, which estimates the state-action value function $Q(s,a,g)$ as $r(s,g)+\gamma \bar V(s',g)$, leading to
\begin{equation}
    \label{eq:IVL}
    \mathcal{L}_{\mathrm{GCIVL}}(\beta)=\mathbb{E}_{(s,s',g)\sim D}\left[L_2^{\tau}(r(s,g)+\gamma V_{\bar\beta}(s',g)-V_{\beta}(s,g))\right].
\end{equation}

Note that using the action-free estimator with expectile regression is optimistically biased in stochastic environments, since it directly regresses towards high-value transitions without using Q-values to marginalize over environment stochasticity. However, it performs well in practice and can be used to measure the relative advantages of intermediate \textit{subgoals} without modification.

\subsection{Offline policy extraction}

The above value iteration algorithm already performs implicit policy improvement. Then, we simply need to extract the optimal policy from the learned value function. To remain inside the convex hull of dataset actions and avoid querying the value function for out-of-distribution actions, we use a form of weighted behavioral cloning called Advantage-Weighted Regression \citep[AWR]{peng_advantage-weighted_2019}:
\begin{align}
    \label{eq:AWR}
    \mathcal{J}_{\pi_\theta}(\theta)=\mathbb{E}_{(s,a,g)\sim D}\left[\exp(\alpha(\bar Q(s,a,g)-V(s,g)))\log \pi_{\theta}(a\mid s,g)\right],
\end{align}

where $\alpha\in[0,\infty)$ is an inverse temperature parameter. Because we do not directly parameterize a critic $Q(s,a,g)$ in GCIVL [Equation \ref{eq:IVL}], we instead estimate the ``one-step'' advantage $\bar Q(s,a,g)-V(s,g)$ as $V(s',g)-V(s,g)$ to reweight the corresponding action in the transition tuple $(s,a,s')$. This estimate is biased in stochastic environments, since it optimistically assumes that the transition from $s$ to $s'$ was caused by action $a$.

\subsection{Policy bootstrapping algorithms}
Policy bootstrapping algorithms are based on the insight that, if an action $a$ is optimal for reaching a subgoal $s_g$ along the optimal trajectory toward goal $g$, then $a$ should also be optimal for reaching $g$. To evaluate the optimality of subgoals $s_g$, we can use a multistep variant of the action-free advantage estimator for GCIVL, where we approximate $A(s, s_g, g)$ as $V(s_g, g) - V(s,g)$.

\textbf{Reinforcement Learning with Imagined Subgoals [RIS]} \citep{chane-sane_goal-conditioned_2021}: To produce subgoals, RIS first learns a generative model $\pi_{\psi}^h(s_g \mid s, g)$ to produce ``imagined'' subgoals $s_g$, and then bootstraps a goal-conditioned flat policy $\pi_{\theta}(a\mid s,g)$ on a short-horizon policy $\pi^{\mathrm{sub}}(a\mid s, s_g)$ conditioned on the imaginary way-points.

The generative model $\pi_{\psi}^h(s_g \mid s, g)$ is optimized to produce subgoals with high advantages while remaining close to the dataset state distribution $\mu_D$, in a manner similar to AWR:
\begin{equation}
\label{eq:RIS_high_policy_AWR}
\mathcal{J}_{\pi_{\psi}^h}({\psi})=
\mathbb{E}_{(s,s_g,g)\sim D}\left[\exp(\alpha A(s_g \mid s,g))\log\pi_{\psi}^h(s_g \mid s,g)\right] .
\end{equation}

\textbf{Subgoal Advantage-Weighted Policy Bootstrapping [SAW]} \citep{zhou_flattening_2025}: SAW introduces a unifying view of RIS and a related hierarchical method \citep{park_hiql_2023} under the ``RL as inference'' framework \citep{levine_reinforcement_2018}. Rather than modeling the optimal subgoal distribution $p^*(s_g \mid s,g)$. SAW samples in-trajectory waypoints in dataset and uses a similar multistep advantage weight to prioritize bootstrapping on high-advantage subgoals.

\begin{equation}
\label{eq:SAW_waypoint_distribution}
p^*(s_g\mid s, g) \propto p^D(s_g\mid s)\exp(A(s_g\mid s,g)).\nonumber
\end{equation}

This eliminates the need for a parametric subgoal generator and replaces Equation \ref{eq:RIS_high_policy_AWR} and \ref{eq:RIS_flat_policy_KL} with

\begin{equation}
\label{eq:SAW_full}
\mathcal{J}^{\text{SAW}}_{\pi_{\theta}}({\theta})=
\mathbb{E}_{(s,a,s_g,g)\sim D}\left[e^{\alpha_1 A(a\mid s,g)}\log \pi_{\theta}(a\mid s,g)-
e^{\alpha_2 A(s_g\mid s,g)}D_{\text{KL}}(\pi_{\theta}\|\pi^{\mathrm{sub}})\right],
\end{equation}

where the target subpolicy $\pi^{\mathrm{sub}}$ is extracted with Equation \ref{eq:low_policy_AWR}, $\pi_{\theta}$ is a shorthand for $\pi_{\theta}^f(\cdot\mid s,g)$, and $\pi^{\mathrm{sub}}$ denotes $\pi^{\mathrm{sub}}(\cdot \mid s_,s_g)$.

Following \citet{zhou_flattening_2025}, we implement the offline variant of RIS and SAW with an independently parameterized short-horizon policy with AWR, instead of using an exponential moving average of the long-horizon policy $\pi^{\theta}$ as in the original implementation by \citet{chane-sane_goal-conditioned_2021}:

\begin{equation}
\label{eq:low_policy_AWR}
\mathcal{J}_{\pi^{\mathrm{sub}}}(\phi)=
\mathbb{E}_{(s,a,s_g)\sim D}\left[\exp(\alpha A(a\mid s,s_g))\log \pi^{\mathrm{sub}}(a\mid s,s_g)\right].
\end{equation}

Finally, we train the long-horizon policy $\pi_{\theta}$ with ``one-step'' transitions while regressing it toward the short-horizon policy $\pi^{\mathrm{sub}}$ conditioned on generated subgoals:

\begin{equation}
\label{eq:RIS_flat_policy_KL}
\mathcal{J}^{\text{RIS}}_{\pi_{\theta}}({\theta})=
\mathbb{E}_{(s,a,g)\sim D, s_g\sim\pi_{\psi}^h(\cdot\mid s,g)}\left[\exp( \alpha A(a\mid s,g))\log \pi_{\theta}(a\mid s,g)  - D_{\text{KL}}(\pi_{\theta}\|\pi^{\mathrm{sub}})\right],
\end{equation}

where the ``one-step'' advantage $A(a\mid s,g)$ is estimated by $V_{\beta_{V}}(s',g)-V_{\beta_{V}}(s,g)$, $\pi_{\theta}$ is a shorthand for $\pi_{\theta}^f(\cdot\mid s,g)$, and $\pi^{\mathrm{sub}}$ denotes $\pi^{\mathrm{sub}}(\cdot \mid s_,s_g)$.

\pagebreak

\section{Additional Figures}

\begin{figure}[h!]
  \begin{center}
  \includegraphics[width=0.9\linewidth]{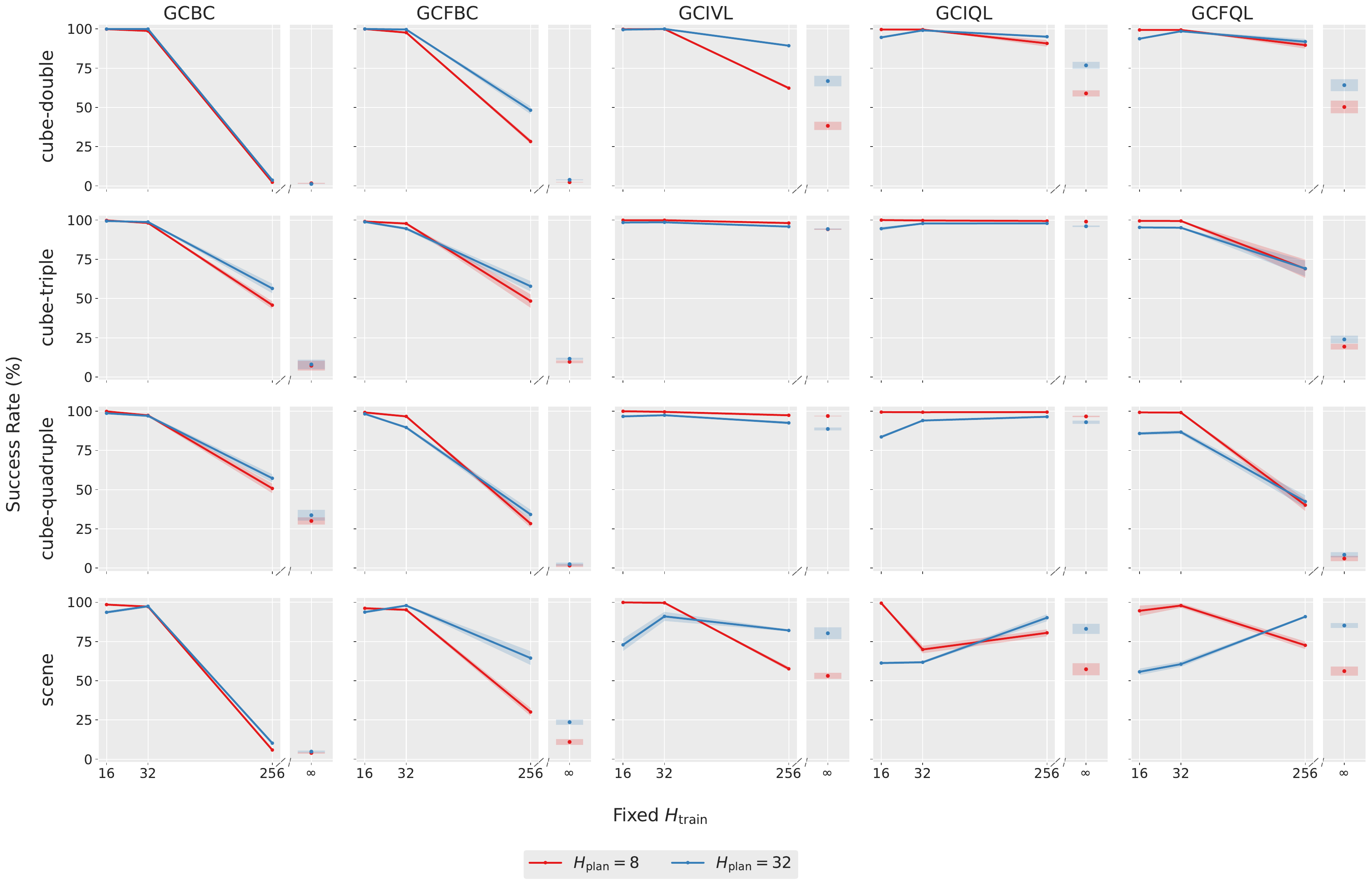}
  \caption{\textbf{Training on distant relabeled goals harms \textit{nearby} goal-reaching capabilities}. We record the success rates across learned policies trained on a uniform distribution of relabeling goal offsets $[1, H_\mathrm{train}]$, where $H_\mathrm{train}=\infty$ corresponds to a uniform distribution over all future in-trajectory states. At evaluation, policies are given optimal $H_\mathrm{plan}$-step subgoals from oracle planners. Success rates are averaged over the last three checkpoints and three seeds.}
  \label{fig:uniform_horizon_oracle_success_all}
  \end{center}
\end{figure}

\begin{figure}[h!]
  \begin{center}
  \includegraphics[width=0.9\linewidth]{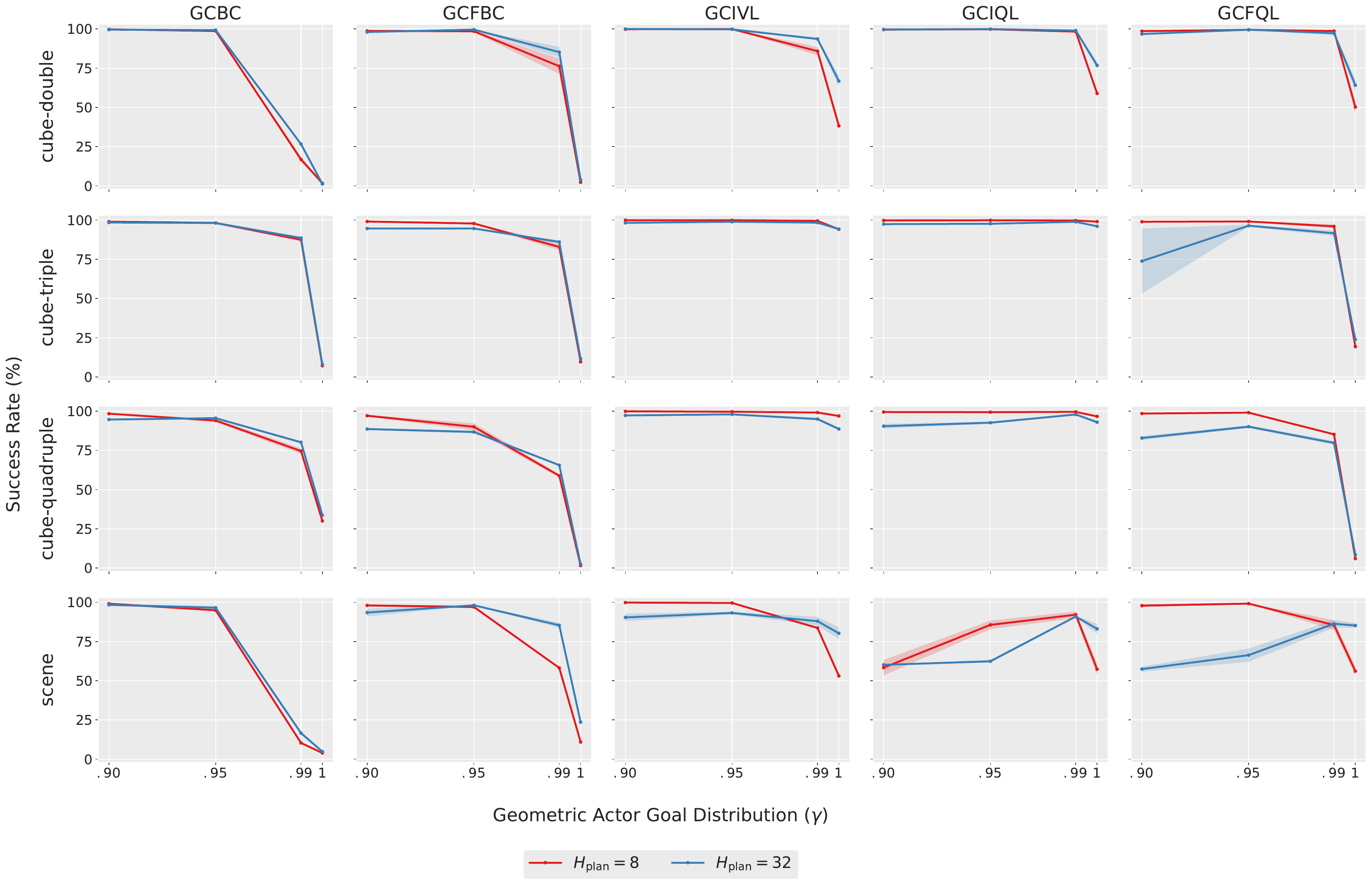}
  \caption{\textbf{Training on a larger proportion of distant relabeled goals harms \textit{nearby} goal-reaching capabilities}. We record the success rates across learned policies trained on a geometric distribution of relabeling goal offsets with discount $\gamma$, where $\gamma=1$ corresponds to a uniform distribution over future states. At evaluation, policies are given optimal $H_\mathrm{plan}$-step subgoals from oracle planners. Success rates are averaged over the last three checkpoints and three seeds.}
  \label{fig:all_geom_horizon_oracle_success}
  \end{center}
\end{figure}

\pagebreak

\begin{figure}[h!]
  \begin{center}
  \includegraphics[width=0.8\linewidth]{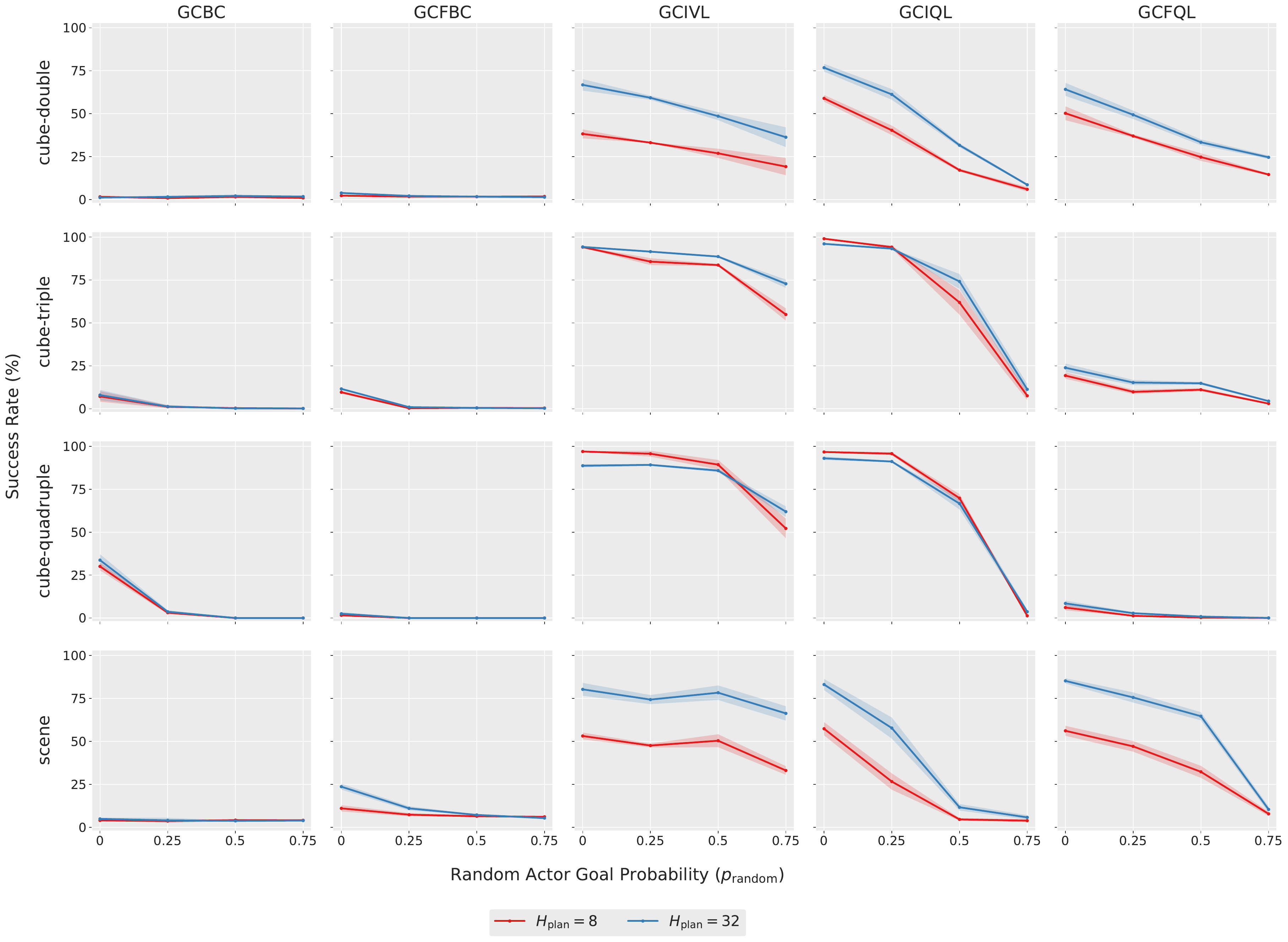}
  \caption{\textbf{Training on more random relabeled goals harms \textit{nearby} goal-reaching capabilities}. We record the success rates across learned policies trained on a mixture of goals either sampled from a uniform distribution over future states or from an arbitrary dataset state with probability $p_\mathrm{rand} \in \{ 0.25, 0.50, 0.75 \}$. At evaluation, policies are given optimal $H_\mathrm{plan}$-step subgoals from oracle planners. Success rates are averaged over the last three checkpoints and three seeds.}
  \label{fig:all_random_oracle_success}
  \end{center}
\end{figure}

\begin{figure}[h!]
  \begin{center}
  \includegraphics[width=0.9\linewidth]{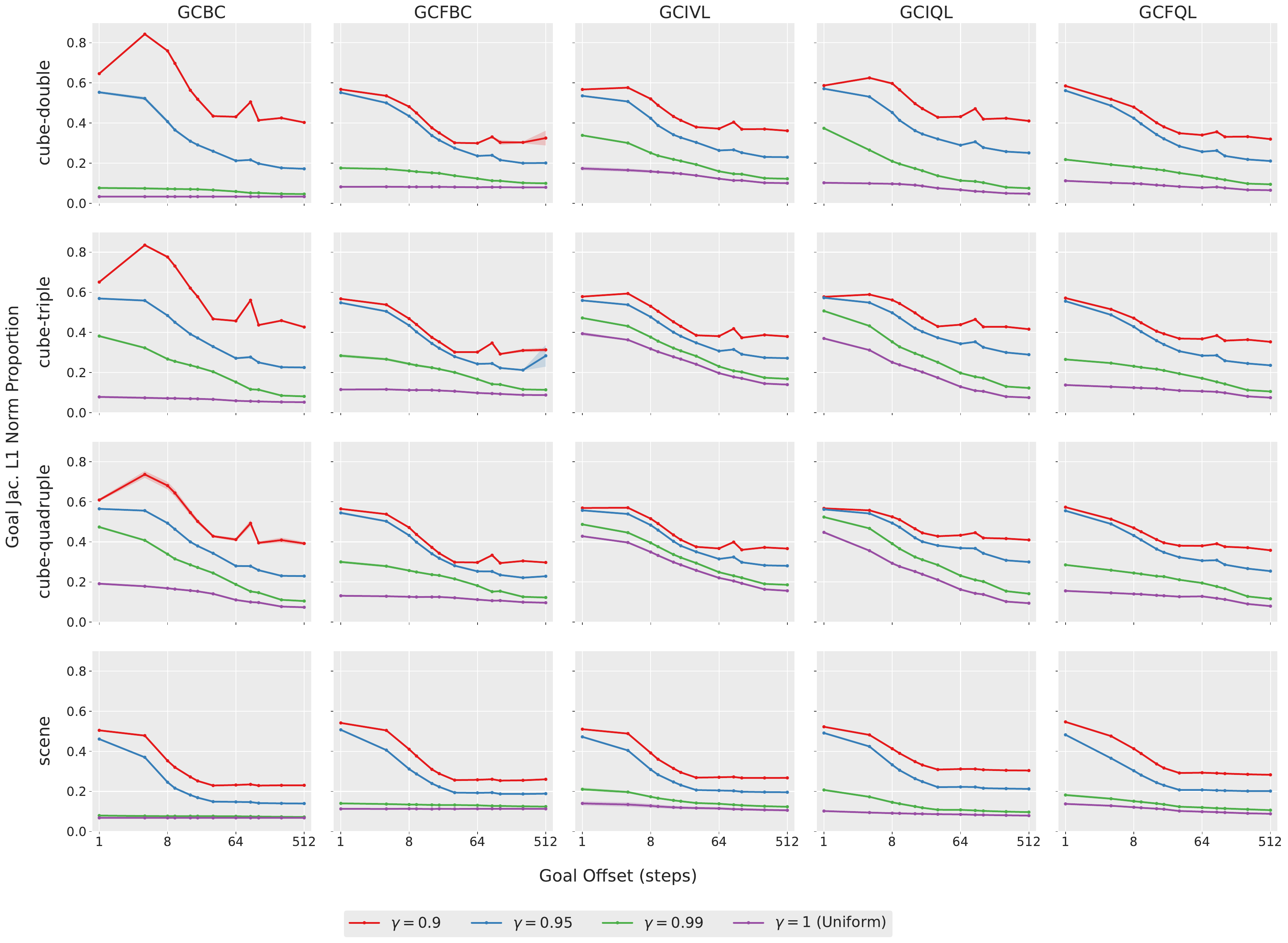}
  \caption{\textbf{Increasing training goal horizon \textit{globally} decreases policy goal sensitivity}. We measure the proportion of the input Jacobian norm corresponding to the goal at fixed state-goal offsets, plotted with $\log_2$ scaling. Results are averaged over the last three training checkpoints, three seeds, and 1024 randomly sampled state-goal pairs per horizon.}
  \label{fig:geom_input_jac_all}
  \end{center}
\end{figure}

\pagebreak

\begin{figure}[h!]
  \begin{center}
  \includegraphics[width=0.9\linewidth]{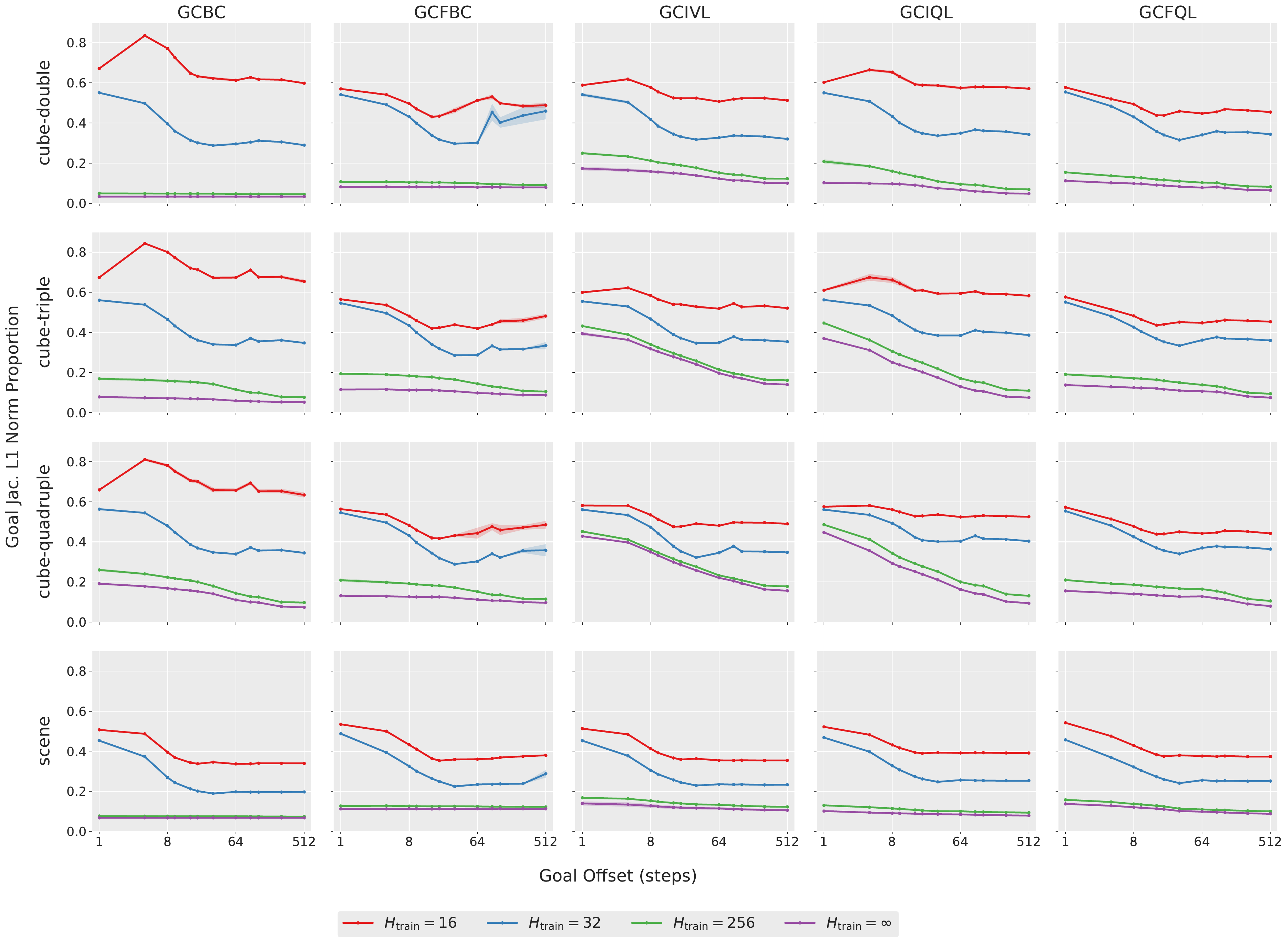}
  \caption{\textbf{Training on distant relabeled goals \textit{globally} decreases policy goal sensitivity}. We measure the proportion of the input Jacobian norm corresponding to the goal at fixed state-goal offsets, plotted with $\log_2$ scaling. Results are averaged over the last three training checkpoints, three seeds, and 1024 randomly sampled state-goal pairs per horizon.}
  \label{fig:uniform_input_jac_all}
  \end{center}
\end{figure}

\begin{figure}[h!]
  \begin{center}
  \includegraphics[width=0.85\linewidth]{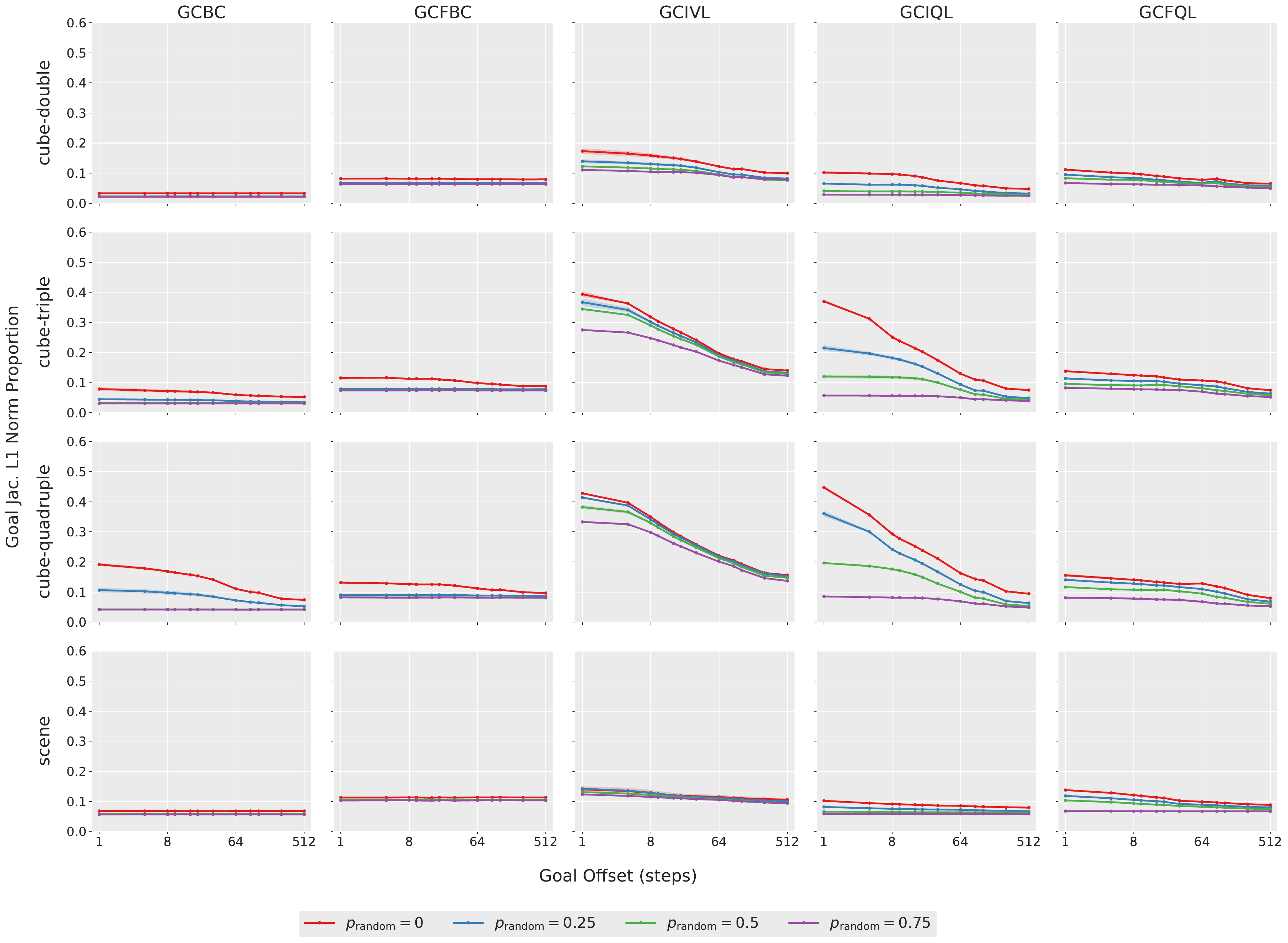}
  \caption{\textbf{Random goal sampling \textit{globally} decreases policy goal sensitivity.} We directly reduce the action-goal MI by replacing in-trajectory goals with randomly sampled dataset states with probability $p_{\text{random}} \in \{0.25, 0.5, 0.75\}$, and measure the goal Jacobian norm proportion at fixed state-goal offsets.}
  \label{fig:all_random_actor_goals}
  \end{center}
\end{figure}

\pagebreak

\begin{figure}[h!]
  \begin{center}
  \includegraphics[width=0.7\linewidth]{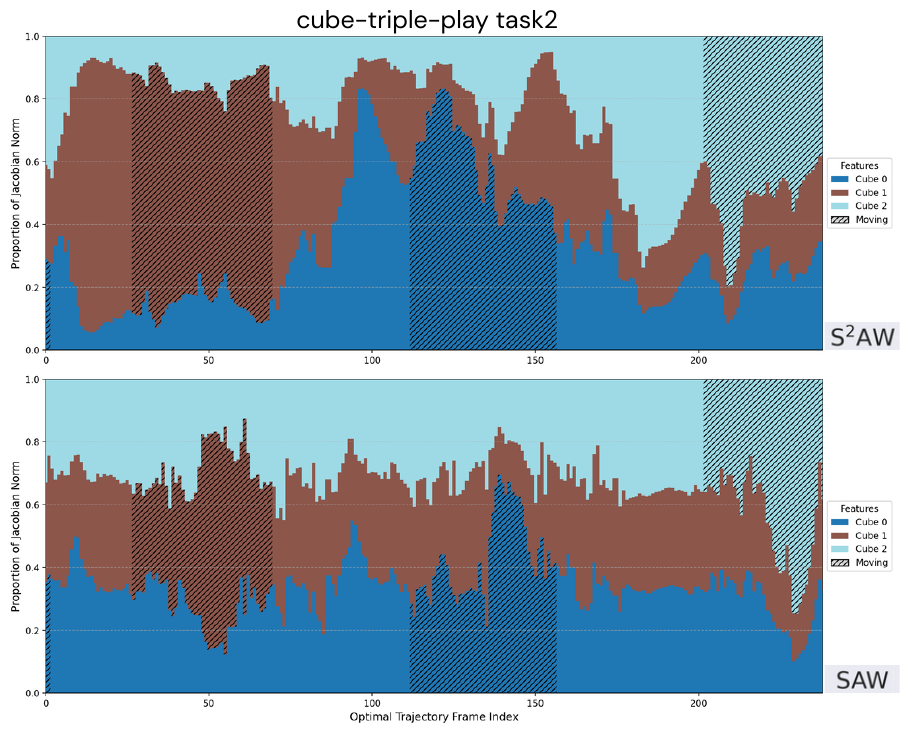}
  \caption{\textbf{Sobolev bootstrapping pays attention to the correct object during optimal trajectories.}
   To enable fair comparisons across later task stages, we plot the input Jacobian norms with respect to each object over an independent scripted optimal trajectory \cite{ahn_option-aware_2025} for the \texttt{cube-triple} triple pick-and-place task. S$^2$AW (top) allocates a higher proportion of the Jacobian norm to the focal cube both during approach and grasp, whereas SAW's (bottom) Jacobian norm only changes after the cube is already grasped. The shaded area denotes when the cube is actively grasped.}
  \label{fig:saw_vs_s2aw_saliency_map}
  \end{center}
\end{figure}

\pagebreak

\section{Hyperparameters}

We provide hyperparameters related to our Sobolev bootstrapping objective in Table \ref{table:sobolev_saw_hyperparameters}. We use value learning hyperparameters from \citet{zhou_flattening_2025} for HIQL, SAW, and S$^2$AW in \texttt{cube-single}, which significantly improves performance. If not listed, all other hyperparameters follow the default values from \citet{park_ogbench_2024} and \citet{zhou_flattening_2025}.

\label{appendix:hparams}
\begin{table}[h]
  \caption{
  \footnotesize
  \textbf{Hyperparameters for SAW with Sobolev bootstrapping.}
  }
  \label{table:sobolev_saw_hyperparameters}
  \centering
  \scalebox{0.7}
  {
  \begin{tabular}{lccccc}
  \toprule
  \textbf{Dataset} & \textbf{Expectile $\tau$} & \textbf{AWR $\alpha$} & \textbf{KLD $\beta$} & \textbf{Subgoal steps $k$} & \textbf{Sobolev weight $\lambda$}\\
  \midrule
  \texttt{pointmaze-medium-navigate-v0} & $0.7$ & $3.0$ & $3.0$ & $25$ & $0.03$\\
  \texttt{pointmaze-large-navigate-v0} & $0.7$ & $3.0$ & $3.0$ & $25$ & $1.00$\\
  \texttt{pointmaze-giant-navigate-v0} & $0.7$ & $3.0$ & $3.0$ & $25$ & $1.00$\\
  \cmidrule{1-6}
  \texttt{antmaze-medium-navigate-v0} & $0.7$ & $3.0$ & $3.0$ & $25$ & $0.03$\\
  \texttt{antmaze-large-navigate-v0} & $0.7$ & $3.0$ & $3.0$ & $25$ & $0.03$\\
  \texttt{antmaze-giant-navigate-v0} & $0.7$ & $3.0$ & $3.0$ & $25$ & $0.30$\\
  \cmidrule{1-6}
  \texttt{humanoidmaze-medium-navigate-v0} & $0.7$ & $3.0$ & $3.0$ & $100$ & $0.03$\\
  \texttt{humanoidmaze-large-navigate-v0} & $0.7$ & $3.0$ & $3.0$ & $100$ & $0.03$\\
  \texttt{humanoidmaze-giant-navigate-v0} & $0.7$ & $3.0$ & $3.0$ & $100$ & $0.03$\\
  \cmidrule{1-6}
  \texttt{cube-single-play-v0} & $0.9$ & $3.0$ & $0.3$ & $10$ & $0.02$\\
  \texttt{cube-double-play-v0} & $0.7$ & $3.0$ & $1.0$ & $10$ & $0.02$\\
  \texttt{cube-triple-play-v0} & $0.7$ & $3.0$ & $1.0$ & $10$ & $0.03$\\
  \texttt{cube-quadruple-play-v0} & $0.7$ & $3.0$ & $1.0$ & $10$ & $0.03$\\
  \cmidrule{1-6}
  \texttt{scene-play-v0} & $0.7$ & $3.0$ & $1.0$ & $10$ & $0.02$\\
  \bottomrule
  \end{tabular}
}
\end{table}

\subsection{Sensitivity Analysis}

We examine sensitivity to the choice of Sobolev weight $\lambda$ in Table \ref{table:sobolev_saw_hyperparameter_sensitivity}. We do not perform additional tuning of the expectile $\tau$, AWR $\alpha$, subgoal steps $k$, and KLD $\beta$, which are common hyperparameters shared by HIQL and SAW. Our results show that smaller weights can be insufficient to realize the benefits of Sobolev training, while larger weights may “oversmooth” in some manipulation environments. Nevertheless, we observe consistently substantial gains in \texttt{cube-triple} and \texttt{antmaze-giant} across all values of $\lambda$.

\begin{table}[h]
\caption{
  \footnotesize
\textbf{Sobolev loss weight sensitivity analysis.} Sobolev bootstrapping is robust to different weighting factors within a reasonable scale (around 0.01 for \texttt{cube-double} and \texttt{scene}). Notably, improvements on \texttt{antmaze-giant} and \texttt{cube-triple} are robust across various weighting scales. Experimental results are averaged across four seeds.
}
\label{table:sobolev_saw_hyperparameter_sensitivity}
\centering
\renewcommand{\arraystretch}{1.0} 
\setlength{\tabcolsep}{3pt}       
\small                            
\scalebox{0.8}
{
\begin{tabular}{lcccc}
\toprule
\textbf{Dataset} & $\lambda=\mathbf{0}$ & $\mathbf{0.001}$ & $\mathbf{0.01}$ & $\mathbf{0.1}$ \\ \midrule

\texttt{cube-double-play-v0}       & $40$ {\tiny $\pm 7$} & $40$ {\tiny $\pm 11$} & ${55}$ {\tiny $\pm 7$} & $35$ {\tiny $\pm 9$}\\
\texttt{cube-triple-play-v0}       & $4$ {\tiny $\pm 2$}  & $9$ {\tiny $\pm 2$}   & ${30}$ {\tiny $\pm 7$} & $23$ {\tiny $\pm 10$} \\
\texttt{scene-play-v0}             & $63$ {\tiny $\pm 6$} & $54$ {\tiny $\pm 10$} & ${71}$ {\tiny $\pm 5$} & $62$ {\tiny $\pm 5$}\\
\texttt{antmaze-giant-navigate-v0} & $73$ {\tiny $\pm 4$} & $76$ {\tiny $\pm 4$}  & $82$ {\tiny $\pm 4$} & ${85}$ {\tiny $\pm 4$} \\ \bottomrule
\end{tabular}
}
\end{table}

\section{Computational Resources}
\label{appendix:compute}
All experiments were run on an internal cluster of NVIDIA 5090 and NVIDIA PRO 6000 Blackwell GPUs. In Table \ref{table:training_time}, We select the highest-dimensional environments and report the hours required to complete a standard 1M-step training run on an NVIDIA 5090 GPU.

\begin{table}[h]
\caption{
  \footnotesize
  \textbf{Training time comparison.}
  We select the highest-dimensional environments and report the hours required to complete a standard 1M-step training run on an NVIDIA 5090 GPU.
}
\label{table:training_time}
\centering
\renewcommand{\arraystretch}{1.0} 
\scalebox{0.8}{
\begin{tabular}{lccc}
\toprule
hours / 1M iters & \textbf{GCIQL} & \textbf{SAW} & \textbf{S$^2$AW} \\ \midrule
\texttt{cube-quadruple-play}         & 0.26 & 0.46 & 0.75 \\
\texttt{scene-play}          & 0.25 & 0.45 & 0.74 \\
\texttt{humanoidmaze-giant-navigate} & 0.26 & 0.46 & 1.75 \\ \bottomrule
\end{tabular}}
\end{table}

\vfill
\pagebreak

\section{Additional Experiments}

\subsection{Architectural ablations}
\label{appendix:arch_ablations}
While we have made the case that differences in state-goal distributions are responsible for the improved properties of the target policy $\pi^\mathrm{sub}$, we note that the experiments in Table \ref{table:results} use different architectures for each policy. The full goal-conditioned policy $\pi$ consists of a three-layer MLP, followed by a length-normalized bottleneck of dimension 10, and another three-layer MLP; the target policy consists only of a three-layer MLP. Then, our results may be explained by a natural simplicity bias introduced by smaller network sizes. To control for this potential confound, we train Sobolev SAW using identical architectures for $\pi^\mathrm{sub}$ and $\pi$, which we denote as \textbf{S$^2$AW}$_\ell$.

\begin{table*}[h!]
  \caption{
  \footnotesize
  \textbf{Architectural ablations to the Sobolev SAW target actor.}
  We report the performance of Sobolev SAW where the target actor has the same architecture as the full policy, averaged over $4$ seeds with standard deviations after the $\pm$ sign.
  Numbers within $5\%$ of the best value in the row are in \textbf{bold}.
  }
  \label{table:architecture_ablation}
  \centering
  \scalebox{0.63}
  {
  \begin{tabular}{lllllllllll}
  \toprule
  \textbf{Environment \:\:} & \textbf{Dataset} & \textbf{GCBC} & \textbf{GCIVL} & \textbf{GCIQL} & \textbf{QRL} & \textbf{CRL} & \textbf{HIQL} & \textbf{SAW} & \textbf{S$^2$AW} & \textbf{S$^2$AW}$_\ell$\\
  \midrule
  \multirow[c]{2}{*}{\texttt{maze-giant}}
   & \texttt{antmaze-giant-navigate-v0} & $0$ {\tiny $\pm 0$} & $0$ {\tiny $\pm 0$} & $0$ {\tiny $\pm 0$} & $14$ {\tiny $\pm 3$} & $16$ {\tiny $\pm 3$} & ${65}$ {\tiny $\pm 5$} & ${73}$ {\tiny $\pm 4$} & $\mathbf{87}$ {\tiny $\pm 3$} & ${80}$ {\tiny $\pm 5$}\\
   & \texttt{humanoidmaze-giant-navigate-v0} & $0$ {\tiny $\pm 0$} & $0$ {\tiny $\pm 0$} & $0$ {\tiny $\pm 0$} & $1$ {\tiny $\pm 0$} & $3$ {\tiny $\pm 2$} & ${12}$ {\tiny $\pm 4$} & ${35}$ {\tiny $\pm 4$} & $\mathbf{38}$ {\tiny $\pm 6$} & ${35}$ {\tiny $\pm 5$}\\
  \cmidrule{1-11}
   \multirow[c]{4}{*}{\texttt{cube}} & \texttt{cube-single-play-v0} & $6$ {\tiny $\pm 2$} & $53$ {\tiny $\pm 4$} & ${68}$ {\tiny $\pm 6$} & $5$ {\tiny $\pm 1$} & $19$ {\tiny $\pm 2$} & $44$ {\tiny $\pm 9$} & ${72}$ {\tiny $\pm 5$} & $\mathbf{86}$ {\tiny $\pm 5$} & ${79}$ {\tiny $\pm 8$}\\
   &  \texttt{cube-double-play-v0} & $1$ {\tiny $\pm 1$} & $36$ {\tiny $\pm 3$} & ${40}$ {\tiny $\pm 5$} & $1$ {\tiny $\pm 0$} & $10$ {\tiny $\pm 2$} & $6$ {\tiny $\pm 2$} & ${40}$ {\tiny $\pm 7$} & ${50}$ {\tiny $\pm 6$} & $\mathbf{55}$ {\tiny $\pm 8$}\\
   &  \texttt{cube-triple-play-v0} & $1$ {\tiny $\pm 1$} & $1$ {\tiny $\pm 0$} & $3$ {\tiny $\pm 1$} & $0$ {\tiny $\pm 0$} & $\mathbf{4}$ {\tiny $\pm 1$} & $3$ {\tiny $\pm 1$} & ${4}$ {\tiny $\pm 2$} & $\mathbf{30}$ {\tiny $\pm 5$} & $\mathbf{30}$ {\tiny $\pm 4$}
   \\
   &  \texttt{cube-quadruple-play-v0} & $0$ {\tiny $\pm 0$} & $0$ {\tiny $\pm 0$} & $0$ {\tiny $\pm 0$} & $0$ {\tiny $\pm 0$} & $0$ {\tiny $\pm 0$} & $0$ {\tiny $\pm 0$} & $0$ {\tiny $\pm 0$} & $\mathbf{3}$ {\tiny $\pm 2$} & ${2}$ {\tiny $\pm 1$}
   \\
\cmidrule{1-11}
  \multirow[c]{1}{*}{\texttt{scene}} & \texttt{scene-play-v0} & $5$ {\tiny $\pm 1$} & $42$ {\tiny $\pm 4$} & ${51}$ {\tiny $\pm 4$} & $5$ {\tiny $\pm 1$} & $19$ {\tiny $\pm 2$} & $38$ {\tiny $\pm 3$} & ${63}$ {\tiny $\pm 6$} & $\mathbf{73}$ {\tiny $\pm 6$} & $\mathbf{72}$ {\tiny $\pm 5$}\\
  \bottomrule
  \end{tabular}
}
\end{table*}

We find that results remain consistent across all manipulation environments where combinatorial generalization is most important, supporting the primary conclusion that the difference in training curricula, and not architectural differences, produces desirable properties of input gradients. Performance does drop slightly in locomotion environments compared to Sobolev SAW with the smaller target actor, suggesting that a smaller actor network does provide some beneficial regularization. Indeed, we observe that Sobolev bootstrapping with a smaller target policy achieves lower SAW validation loss than its counterpart.

\subsection{Extending Sobolev bootstrapping to RIS}
\label{appendix:ris_results}

Adding Sobolev bootstrapping to RIS \citep{chane-sane_goal-conditioned_2021} also yields significant gains, albeit less consistently. In our implementation of RIS, subgoals generated by $\pi_{\psi}^h$ lie in a lower-dimensional latent space rather than the raw observation space, making it impractical to compute advantages directly from the learned value function $V_{\beta_V}$. Thus, we adopt a constant weighting scheme for the Sobolev bootstrapping loss, under the assumption that all generated subgoals are near-optimal. If the subgoal generator proposes suboptimal waypoints, then the uniformly weighted Sobolev term in RIS can force the flat policy to match noisy input gradients, which can degrade its performance.

\begin{table}[h!]
  \caption{
  \footnotesize
  \textbf{Evaluating RIS with Sobolev bootstrapping.}
  We evaluate the effectiveness of Sobolev bootstrapping on RIS by comparing it against the numbers reported in \citet{zhou_flattening_2025}. Results are averaged over 4 seeds, and other settings are identical to Table \ref{table:results}.
  }
  \label{table:ris_sobolev_results}
  \centering
  \scalebox{0.8}
  {
  \begin{tabular}{llcc}
  \toprule
  \textbf{Dataset} & \textbf{RIS} & \textbf{Sobolev RIS}\\
  \midrule
  \texttt{pointmaze-medium-navigate-v0} & $88$ {\tiny $\pm 6$} & $\mathbf{96}$ {\tiny $\pm 1$}\\
  \texttt{pointmaze-large-navigate-v0} & $63$ {\tiny $\pm 13$} & $\mathbf{75}$ {\tiny $\pm 3$}\\
  \texttt{pointmaze-giant-navigate-v0} & $57$ {\tiny $\pm 12$} & $\mathbf{67}$ {\tiny $\pm 5$}\\
  \cmidrule{1-3}
  \texttt{antmaze-medium-navigate-v0} & $\mathbf{96}$ {\tiny $\pm 1$} & $\mathbf{97}$ {\tiny $\pm 1$}\\
  \texttt{antmaze-large-navigate-v0} & $\mathbf{89}$ {\tiny $\pm 3$} & $\mathbf{90}$ {\tiny $\pm 1$}\\
  \texttt{antmaze-giant-navigate-v0} & $\mathbf{65}$ {\tiny $\pm 4$} & $\mathbf{69}$ {\tiny $\pm 0$}\\
  \cmidrule{1-3}
  \texttt{humanoidmaze-medium-navigate-v0} & $73$ {\tiny $\pm 5$} & $\mathbf{84}$ {\tiny $\pm 7$}\\
  \texttt{humanoidmaze-large-navigate-v0} & $21$ {\tiny $\pm 7$} & $\mathbf{47}$ {\tiny $\pm 2$}\\
  \texttt{humanoidmaze-giant-navigate-v0} & $3$ {\tiny $\pm 2$} & $\mathbf{13}$ {\tiny $\pm 3$}\\
  \cmidrule{1-3}
  \texttt{cube-single-play-v0} & $\mathbf{81}$ {\tiny $\pm 6$} & $59$ {\tiny $\pm 3$}\\
  \texttt{cube-double-play-v0} & $36$ {\tiny $\pm 4$} & $\mathbf{41}$ {\tiny $\pm 2$}\\
  \texttt{cube-triple-play-v0} & $\mathbf{3}$ {\tiny $\pm 2$} & $\mathbf{3}$ {\tiny $\pm 0$}\\
  \texttt{cube-quadruple-play-v0} & $\mathbf{0}$ {\tiny $\pm 0$} & $\mathbf{0}$ {\tiny $\pm 0$}\\
  \cmidrule{1-3}
  \texttt{scene-play-v0} & $\mathbf{64}$ {\tiny $\pm 7$} & $\mathbf{66}$ {\tiny $\pm 1$}\\
  \bottomrule
  \end{tabular}
}
\end{table}

Furthermore, it is well-known that planning with a latent subgoal representation extracted from an intermediate layer of the value function, as done in \citet{park_hiql_2023} and our implementation of RIS, is suboptimal and can significantly hamper performance in larger environments such as \texttt{antmaze-giant} and \texttt{humanoidmaze-giant} \citep{zhou_flattening_2025,ahn_option-aware_2025}. Although not a focus of our work, we expect that improvements to the high-level policy, such as modern generative models and new representation learning objectives, may significantly improve RIS's performance.

\subsection{Additional horizon reduction baselines}
\label{appendix:additional_baselines}
Existing horizon-reduction approaches address one of two issues: compounding bias in off-policy value learning, e.g., $n$-step TD \citep{park_horizon_2025}, QC \citep{li_reinforcement_2025, li_decoupled_2026}, OTA \citep{ahn_option-aware_2025}; or noisy advantage estimates during policy extraction, e.g., HIQL \citep{park_hiql_2023}, SAW \citep{zhou_flattening_2025}. In contrast, our work identifies an additional bottleneck within the actor itself, related to decreased policy sensitivity to goal information scaling with horizon.

For completeness, we compare our method against OTA \citep{ahn_option-aware_2025}, DQC \citep{li_decoupled_2026}, and SHARSA \citep{park_horizon_2025}. Because the original OTA paper did not report results on state-based manipulation tasks, we use the official implementation and sweep the abstraction factor $n \in \{5, 10, 20\}$. Furthermore, adapting DQC and SHARSA to our standard dataset is non-trivial, as they were evaluated using larger datasets with oracle representations. To ensure a fair comparison, we adopt the documented hyperparameters and use the standard batch and model sizes consistent with the OGBench suite. Our results demonstrate that OTA, DQC, and SHARSA underperform with the standard datasets.

\begin{table}[h]
\caption{
  \footnotesize
  \textbf{Additional horizon reduction baselines}. 
   The additional baselines do not readily generalize to the standard datasets and perform poorly. All results are averaged across 4 seeds.
}
\centering
\renewcommand{\arraystretch}{1.0} 
\scalebox{0.8}{
\begin{tabular}{lcccccc}
\toprule
\textbf{Environment} & \textbf{OTA n=5} & \textbf{OTA n=10} & \textbf{OTA n=20} & \textbf{SHARSA} & \textbf{DQC} & \textbf{S$^2$AW} \\ \midrule
\texttt{antmaze-giant-navigate} & 77{\scriptsize $\pm 4$} & - & - & 72{\scriptsize $\pm 1$} & 25{\scriptsize $\pm 6$} & 87{\scriptsize $\pm 3$} \\
\texttt{cube-double-play}       & 2{\scriptsize $\pm 0$}  & 3{\scriptsize $\pm 1$} & 2{\scriptsize $\pm 1$} & 5{\scriptsize $\pm 1$} & 24{\scriptsize $\pm 4$} & 57{\scriptsize $\pm 6$} \\
\texttt{cube-triple-play}       & 1{\scriptsize $\pm 0$}  & 0{\scriptsize $\pm 0$} & 1{\scriptsize $\pm 0$} & 2{\scriptsize $\pm 1$} & 11{\scriptsize $\pm 4$} & 30{\scriptsize $\pm 5$} \\ \bottomrule
\end{tabular}}
\end{table}

\subsection{Uniform Jacobian regularization}
\label{appendix:l1_l2_jacobian_reg}
\begin{table}[H]
\caption{
  \footnotesize
  \textbf{Uniform Jacobian regularization.} 
  We apply $\ell_1$ and $\ell_2$ norm penalization with coefficient $0.3$ to the SAW baseline. We observe that $\ell_1$ penalization is more effective, leading to notable performance gains in \texttt{cube-double} and \texttt{scene} environments. Results are averaged across four seeds and five evaluation tasks.
}
\label{table:jac_reg}
\centering
\renewcommand{\arraystretch}{1.0} 
\scalebox{0.8}{
\begin{tabular}{lcccc}
\toprule
 & \textbf{SAW} & \textbf{SAW$_{\ell_1}$} & \textbf{SAW$_{\ell_2}$} & \textbf{S$^2$AW} \\ \midrule
\texttt{antmaze-giant-navigate} & 73{\scriptsize $\pm 4$} & 72{\scriptsize $\pm 3$}  & 74{\scriptsize $\pm 2$} & 87{\scriptsize $\pm 3$} \\
\texttt{cube-double-play}         & 40{\scriptsize $\pm 7$} & 47{\scriptsize $\pm 10$} & 37{\scriptsize $\pm 7$} & 57{\scriptsize $\pm 6$} \\
\texttt{cube-triple-play}         & 4{\scriptsize $\pm 2$}  & 5{\scriptsize $\pm 4$}   & 4{\scriptsize $\pm 1$}  & 30{\scriptsize $\pm 5$} \\
\texttt{scene-play}         & 63{\scriptsize $\pm 6$} & 69{\scriptsize $\pm 8$}  & 62{\scriptsize $\pm 2$} & 74{\scriptsize $\pm 6$} \\ \bottomrule
\end{tabular}}
\end{table}

\end{document}